\documentclass{article}
\usepackage{kr}

\usepackage{times}
\usepackage{soul}
\usepackage{url}
\usepackage[hidelinks]{hyperref}
\usepackage[utf8]{inputenc}
\usepackage[small]{caption}
\usepackage{graphicx}
\usepackage{amsmath}
\usepackage{amsthm}
\usepackage{booktabs}
\usepackage{algorithm}
\usepackage{savesym}
\savesymbol{Bbbk}
\savesymbol{openbox}

\usepackage{amssymb}
\usepackage{amsthm}
\usepackage{MnSymbol}
\usepackage{color}
\usepackage{url}
\usepackage{multirow}
\usepackage{xcolor}
\usepackage{dsfont}
\usepackage{algorithm}
\usepackage[noend]{algorithmic}

\restoresymbol{NEW}{Bbbk}
\restoresymbol{NEW}{openbox}

\usepackage[switch]{lineno}
\usepackage{tablefootnote}

\usepackage{subcaption}

\newcommand{\FT}[1]{\textcolor{black}{#1}} 

\newcommand{\A}[1]{\textcolor{black}{#1}}
\newcommand{\AR}[1]{\textcolor{black}{#1}}
\newcommand{\todo}[1]{\textcolor{magenta}{#1}}
\newcommand{\delete}[1]{\textcolor{olive}{$\times$ #1}}
\newcommand{\cut}[1]{\textcolor{orange}{%#1
}}

\def\resp{resp.}
\def\Iff{iff}

\def\RR{RR}
\def\CR{CR}
\def\PR{PR}
\def\CA{CA}

\def\public{%public?common?explanatory?shared?
exchange}

\newcommand{\Args}{\ensuremath{\mathcal{X}}}
\newcommand{\Atts}{\ensuremath{\mathcal{A}}}
\newcommand{\Supps}{\ensuremath{\mathcal{S}}}
\newcommand{\Agents}{\ensuremath{{AG}}}
\newcommand{\SpeakerM}{\ensuremath{\mathcal{C}}}

\newcommand{\Range}{\ensuremath{\mathbb{I}}}
\newcommand{\FW}{\ensuremath{\mathcal{F}}}
\newcommand{\BAF}{\ensuremath{\mathcal{B}}}
\newcommand{\QBAF}{\ensuremath{\mathcal{Q}}}
\newcommand{\SF}{\ensuremath{\sigma}}
\newcommand{\BS}{\ensuremath{\tau}}
\newcommand{\RPos}{\ensuremath{\Range^+}}
\newcommand{\RNeg}{\ensuremath{\Range^-}}
\newcommand{\RNeu}{\ensuremath{\Range^0}}
\newcommand{\Stance}{\ensuremath{\Sigma}}
\newcommand{\Conv}{\ensuremath{\mathcal{C}}}

\newcommand{\AGa}{\ensuremath{\alpha}}
\newcommand{\AGb}{\ensuremath{\beta}}
\newcommand{\AGh}{\ensuremath{\eta}}
\newcommand{\AGm}{\ensuremath{\mu}}

\newcommand{\AGx}{\ensuremath{\phi}}
\newcommand{\AGy}{\ensuremath{\chi}}
\newcommand{\Int}{\ensuremath{%p
x}}
\newcommand{\Uni}{\ensuremath{u}}
\newcommand{\eff}{\ensuremath{\epsilon}}
\newcommand{\Pros}{\ensuremath{\mathsf{pro}}}
\newcommand{\Cons}{\ensuremath{\mathsf{con}}}

\newcommand{\argpaths}{\mathsf{paths}}

\newcommand{\BAFi}{\ensuremath{\BAF_\Int}}

\newcommand{\QBAFa}{\ensuremath{\QBAF_\AGa}}
\newcommand{\QBAFb}{\ensuremath{\QBAF_\AGb}}
\newcommand{\QBAFh}{\ensuremath{\QBAF_\AGh}}
\newcommand{\QBAFm}{\ensuremath{\QBAF_\AGm}}

\newcommand{\ViewA}{\ensuremath{\QBAF_{\AGa v}}}

\newcommand{\ArgsA}{\ensuremath{\Args_\AGa}}
\newcommand{\ArgsB}{\ensuremath{\Args_\AGb}}
\newcommand{\ArgsH}{\ensuremath{\Args_\AGh}}
\newcommand{\ArgsM}{\ensuremath{\Args_\AGm}}
\newcommand{\ArgsI}{\ensuremath{\Args_\Int}}

\newcommand{\ArgsVA}{\ensuremath{\Args_{\AGa v}}}

\newcommand{\AttsA}{\ensuremath{\Atts_\AGa}}

\newcommand{\AttsH}{\ensuremath{\Atts_\AGh}}
\newcommand{\AttsM}{\ensuremath{\Atts_\AGm}}
\newcommand{\AttsI}{\ensuremath{\Atts_\Int}}

\newcommand{\AttsVA}{\ensuremath{\Atts_{\AGa v}}}

\newcommand{\SuppsA}{\ensuremath{\Supps_\AGa}}

\newcommand{\SuppsH}{\ensuremath{\Supps_\AGh}}
\newcommand{\SuppsM}{\ensuremath{\Supps_\AGm}}
\newcommand{\SuppsI}{\ensuremath{\Supps_\Int}}

\newcommand{\SuppsVA}{\ensuremath{\Supps_{\AGa v}}}

\newcommand{\SFa}{\ensuremath{\SF_\AGa}}
\newcommand{\SFb}{\ensuremath{\SF_\AGb}}
\newcommand{\SFh}{\ensuremath{\SF_\AGh}}
\newcommand{\SFm}{\ensuremath{\SF_\AGm}}

\newcommand{\BSa}{\ensuremath{\BS_\AGa}}

\newcommand{\BSh}{\ensuremath{\BS_\AGh}}
\newcommand{\BSm}{\ensuremath{\BS_\AGm}}

\newcommand{\BSva}{\ensuremath{\BS_{\AGa v}}}

\newcommand{\RangeA}{\ensuremath{\Range_\AGa}}

\newcommand{\RangeH}{\ensuremath{\Range_\AGh}}
\newcommand{\RangeM}{\ensuremath{\Range_\AGm}}

\newcommand{\RPosA}{\ensuremath{\RPos_\AGa}}
\newcommand{\RNegA}{\ensuremath{\RNeg_\AGa}}
\newcommand{\RNeuA}{\ensuremath{\RNeu_\AGa}}

\newcommand{\RPosH}{\ensuremath{\RPos_\AGh}}
\newcommand{\RNegH}{\ensuremath{\RNeg_\AGh}}
\newcommand{\RNeuH}{\ensuremath{\RNeu_\AGh}}
\newcommand{\RPosM}{\ensuremath{\RPos_\AGm}}
\newcommand{\RNegM}{\ensuremath{\RNeg_\AGm}}
\newcommand{\RNeuM}{\ensuremath{\RNeu_\AGm}}

\newcommand{\StanceA}{\ensuremath{\Stance_\AGa}}
\newcommand{\StanceB}{\ensuremath{\Stance_\AGb}}
\newcommand{\StanceH}{\ensuremath{\Stance_\AGh}}
\newcommand{\StanceM}{\ensuremath{\Stance_\AGm}}

\newcommand{\LB}{\ensuremath{\mathcal{L}}}

\newcommand{\turn}{\ensuremath{\pi}}

\newcommand{\arga}{\ensuremath{a}}
\newcommand{\argb}{\ensuremath{b}}
\newcommand{\argc}{\ensuremath{c}}
\newcommand{\argd}{\ensuremath{d}}
\newcommand{\arge}{\ensuremath{e}}
\newcommand{\argf}{\ensuremath{f}}

\newtheorem{example}{Example}

\newtheorem{definition}{Definition}
\newtheorem{proposition}{Proposition}

\newtheorem{property}{Property}

\title{%\textcolor{red}{CONFIDENTIAL - PLEASE DO NOT DISTRIBUTE\\ \quad \\}
Interactive Explanations by Conflict Resolution via Argumentative Exchanges
 }

\author{%
Antonio Rago\and
Hengzhi Li\And
Francesca Toni \\
\affiliations
Department of Computing, Imperial College London, UK\\
\emails
\{a.rago, hengzhi.li21, ft\}@imperial.ac.uk
}

\begin{document}

\maketitle

\begin{abstract}
As the field of {explainable AI} (XAI) is maturing, %the
calls for 
interactive explanations for (the outputs of) AI models
are growing,  but the state-of-the-art predominantly focuses on %explanations that are delivered statically% and can host only shallow reasoning
\emph{static explanations}.
In this paper, we focus instead on interactive explanations framed as \emph{conflict resolution}  between agents (i.e. AI models and/or humans) % offers a novel perspective on, and alternative methods for, this task. W}e 
by leveraging on {computational argumentation}%, which we believe is \delete{uniquely}\AR{particularly} well placed for harbouring, analysing and delivering the information in interactive explanations
. Specifically, we define \emph{Argumentative eXchanges (AXs)} %, a mechanism
for %representing a dynamic form of 
dynamically sharing, in multi-agent systems, information harboured in individual agents' \emph{quantitative bipolar argumentation frameworks} towards resolving conflicts amongst the agents. We then deploy %argumentative exchanges 
AXs in the XAI setting in which a machine and a human interact about the machine's predictions% when the human and the machine disagree
.  
We identify and assess several theoretical properties characterising %argumentative exchanges 
AXs that are suitable for XAI. %\AR{theoretically how they satisfy various desirable properties for this setting.} 
Finally, we instantiate %our argumentative exchanges 
AXs for XAI by
defining various agent behaviours, e.g. capturing %human-like phenomena, such as 
counterfactual patterns of reasoning in machines and highlighting the effects of cognitive biases in humans. We
show experimentally (in a simulated environment) %and theoretically 
the comparative
advantages of these behaviours %in the context of argumentative exchanges 
in terms of %supporting interactivity and non-shallow reasoning
conflict resolution, and show that
the strongest argument may not always be the most effective.
\end{abstract}

%\todo{``Finally, the notation seems sometimes a bit heavy, and intuitions are sometimes missing (for instance I think the different behaviours could be presented more simply)’’ We aimed at a fully reproducible framework. We will add intuitions and try to simplify where possible. \A{I would ignore this, I think we've done enough?}}

\section{Introduction}

The need for interactivity in  explanations of the outputs of AI models has long been called for \cite{Cawsey_91}, and the recent wave of explainable AI (XAI) has given rise to renewed urgency in the matter. In \cite{Miller_19}, it is stated that explanations need to be  social, and thus for machines to truly explain themselves, they must be interactive, %noting that XAI should not just be `
so that XAI is not just ``more AI'', but a human-machine interaction problem.
Some have started exploring explanations as dialogues~\cite{Lakkaraju_22X}
, while several are exploring forms of interactive machine learning for model debugging~\cite{Teso_23}. 
%\delete{There are also clear indications in application that it would be helpful if humans could argue with machines \cite{Hirsch_18}.} 
It has also been claimed that it is our responsibility to create machines which can argue with humans \cite{Hirsch_18}.
However, despite the widespread acknowledgement of the need for interactivity, typical approaches to XAI deliver \emph{``static'' explanations}, whether they be based on feature attribution (e.g. as in \cite{Lundberg_17}), counterfactuals (e.g. as in \cite{Wachter_17X}) or other factors such as prime implicants (e.g. as in \cite{Shih_18,Ignatiev_19}).
These explanations typically focus exclusively on aspects of the input %\AR{features}  I HAVE USED ASPECTS AS IT SEEMS MORE GENERIC....
deemed responsible (in different ways, according to the method used) for the outputs of the explained AI model, and offer %limited opportunities 
little opportunity for interaction.
%usually consisting of some categorisation of the feature variables for a given (or modified) input and do not %offer the opportunity for users to provide feedback, 
%allow for accommodating human feedback.
For illustration, consider a recommender system providing positive and negative evidence drawn from input features
as an explanation for a movie %\AR{being poorly rated} WHY DO WE CARE HERE? IT IS A DISTRACTION FROM THE MAIN POINT
recommendation
to a user: this form of explanation is static in that %do not contain mechanisms for interactively handling conflicts 
it does not support interactions between the system and the user, e.g. if the latter %\delete{wants to enquire about or \todo{I would avoid mentioning this since we don't allow for enquiry}} 
disagrees with the role of the input features in the explanation towards the recommendation, or %\delete{if the user disagrees on} 
with the system's recommendation itself. % FT THIS DOES NOT FIT HERE - WE NEED TO MENTION THIS IN THE SECTION WHERE THESE THINGS ARE DICUSSED Further, such explanations provide limited capability for addressing known challenges in XAI such as providing human-like reasoning \cite{Miller_19} or accounting for cognitive biases in users \cite{Bertrand_22}.

\iffalse FT: DANGEROUS SENTENCE, IT RAISES EXPECTATIONS - WE DO NOT DO HUMAN EVALUATION AT ALL
Further, these methods are often evaluated wrt a set of (undoubtedly useful) machine-centric metrics, e.g. size, computation time or completeness, without any consideration of the %human factors which play a role in social explanation, e.g.   
role that humans can play in the explanatory process \todo{\cite{?}}, the importance of human-like reasoning mechanisms \cite{Miller_19} or the accounting for cognitive biases present in users of XAI \cite{Bertrand_22}. 
\todo{The next sentence is confrontational and unnecessary? FOCUS ON WHAT WE DO INSTEAD}We argue that these trends are leading the field of XAI towards a rigid and computational view of explanation, \delete{whereas} \AR{and that} in order to achieve trustworthy AI, explanations for the outputs of AI models must be delivered %via reasoning mechanisms which 
to
closely align with those of humans %to allow 
for interactivity.
\fi

%\todo{model adjustment/alignment (xai as reconciliation)}

A %recent 
parallel research direction %in XAI
focuses on %the use of 
\emph{argumentative explanations} for AI models of various types (see \cite{Cyras_21,Vassiliades_22} for recent overviews), 
often motivated by the appeal of argumentation in  explanations amongst humans, e.g. as in \cite{Antaki_92}, within the broader view that XAI should take findings from %studies in 
the social sciences into account~\cite{Miller_19}.
%NOT REALLY - MILLER NEVER SAID USE COMPUTATIONAL ARGUMENTATION Such explanations have been advocated in the social sciences, not least due to argumentation's capability to support challenges from humans \cite{Miller_19} and since the majority of statements made in explanations are, in fact, argumentative in that they not only report causes but also provide reasoning to justify why causes hold (or are thought to hold) \cite{Antaki_92}. 
Argumentative explanations in XAI employ \emph{computational argumentation} (see \cite{AImagazine17,handbook} for overviews)%as a knowledge representation and reasoning mechanism 
%for defining explanations
, leveraging upon (existing or novel)  argumentation frameworks, semantics and properties.

Argumentative explanations seem well suited to support interactivity
when the mechanics of AI models  can be abstracted away %in argumentative terms
argumentatively (e.g.  as %in the case of 
for 
some recommender systems  \cite{Rago_18} or neural networks \cite{Albini_21,Potyka_21}).
For illustration, consider  the case of a movie review aggregation system, as in \cite{Cocarascu_19}, and assume that its recommendation %\delete{for}
of a %\delete{low rating for some} 
movie $x$ and its reasoning therefor
%is explained, to start with, by means of the 
can be %\delete{abstracted away in terms of}
represented by the \emph{bipolar argumentation framework} (BAF) \cite{Cayrol:05} 
$\langle \Args, \Atts, \Supps \rangle$ with \emph{arguments} $\Args \!= \!\{ \arge, m_1, m_2 %\delete{, \argc} 
\}$, \emph{attacks} %(i.e. negative relations between arguments)} THIS IS KR
$\Atts \!= \! \emptyset %\{(\arga, \arge)\}
$ and \emph{supports} %(i.e. positive relations between arguments)} 
$\Supps \!=\! \{(m_1, \arge), (m_2, m_1)\}$ (see left of Figure~\ref{fig:toy} for a graphical visualisation).  
Then, %as in the case of static explanations, 
by %\delete{attacking}
supporting $\arge$, $m_1$ (statically) conveys shallow evidence for the output (i.e. %\delete{a low rating for $m$}
movie $x$ being recommended).
%\todo{why are we making our life difficult? it seems counterintuitive that the movie is not recommended...why can't we have that the movie is recommended but not strongly, so that we have a support from a which is attacked by c? and may be b attacks m? it would be easier to explain and less convoluted \AR{we could do this but then we need to change the examples throughout, it's doable but I just need to spend some time making sure everything fits} 
%\delete{In addition, $\argc$ provides evidence for $\arga$, going beyond the shallow nature of state-of-the-art explanations in XAI.  Now, if the .}
Argumentative explanations may go beyond the shallow nature of state-of-the-art explanations by facilitating dynamic, interactive explanations, e.g. by allowing a human explainee who does not agree with the machine's output or the evidence it provides (in other words, there is a \emph{conflict} between the machine and the human) %\delete{, the human can} 
to provide feedback (in Figure~\ref{fig:toy}, 
by introducing %\delete{support $(\argb,\arge)$ and} 
\cut{either/both of the} attacks $(h_1,\arge)$ or $(h_2,m_1)$%\delete{, \resp}
), while also allowing for the system to provide additional information (in Figure~\ref{fig:toy}, 
by introducing the support %\delete{$(\argc,\arga)$}
$(m_2,m_1)$).
The resulting interactive explanations can be seen as a \emph{conflict resolution} process, 
%may naturally arise from dialogues between the system and the human, 
e.g. as in \cite{Raymond_20}.
%their potential for giving rise to genuinely interactive explanations has not been explored to date. 
Existing approaches %(including \delete{the aforementioned}\AR{that of} \cite{Raymond_20}) 
focus on %the dialogical aspects of interactive explanations, often in 
specific settings. Also, although the need for studying properties of explanations is well-acknowledged \cut{in general} (e.g. see~\cite{Sokol_20,Amgoud_22})% and for argumentative explanations (e.g. see~\cite{Amgoud_22})
, to the best of our knowledge properties of  
interactive explanations, e.g. relating to how well they represent %the existence and possible resolution of 
and %\delete{possibly}
resolve any conflicts, have been neglected to date.
\begin{figure}
    \centering
    \includegraphics[width=1\linewidth]{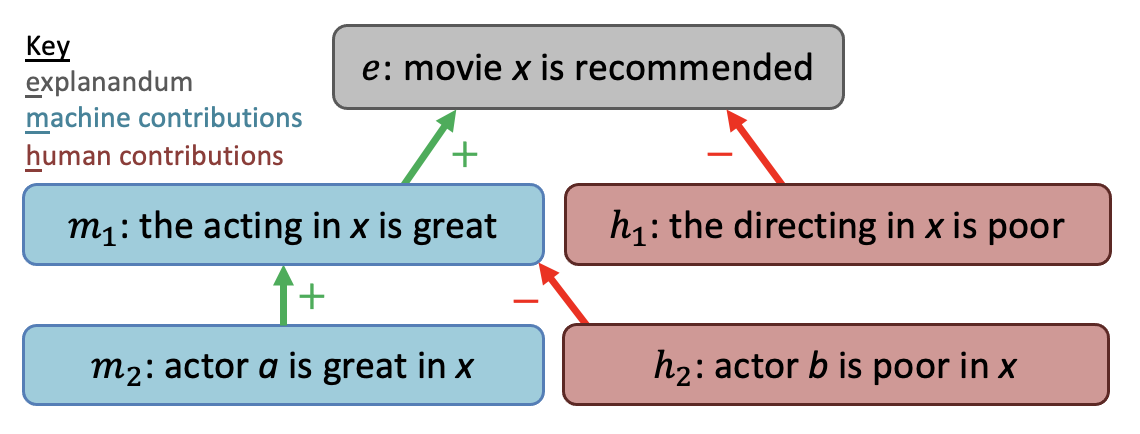}
    \caption{%A BAF $\langle \Args, \Atts, \Supps \rangle$ with arguments $\Args = \{ \arge, \arga, \argb,\argc  \}$,  $\Atts = \{(\arga, \arge)\}$, $\Supps = \{(\argb, \arge), (\argc,\arga)\}$, visualised as a graph. Here, $\arge$ may stand for ``the movie is poorly rated'', $\arga$ for ``the acting was poor'' and $\argb$ for ``the directing was great'', and $\argc$ may be of the form ``reviewers said [\ldots] about  the acting''. 
An argumentative explanation for a review aggregation system, %\delete{resulting from}
amounting to the interactions between a machine  %interactively explains its 
and a human sharing their 
reasoning following a recommendation for 
 $x$.
 %\todo{``From the title, I expected to have a clear idea of what could be an explanation in an interactive setting, but I may have missed something.  For instance, what is considered an explanation in Figure 2 at time 3? The common public tree? Or the private one for the human agent?’’ Your intuition about the “common public tree” as explanations for humans is correct. We will emphasise this, specifically for Figs1 and 3.}
 }
    \label{fig:toy}
\end{figure}
%
%In this paper, we 
We 
fill these gaps by providing a general argumentative framework for interactive explanations as conflict resolution, as well as properties and instantiations thereof, backed by simulated experiments. Specifically\cut{,after covering related work (§\ref{sec:related}) and the necessary preliminaries (§\ref{sec:preliminaries}), we make %the following 
contributions as follows}:
\begin{itemize}
\item We %leverage on computational argumentation to define the novel concept of 
define \emph{Argumentative eXchanges} (AXs, §\ref{sec:exchanges}), in which agents, whose \cut{internal} reasoning is represented as 
\emph{quantitative bipolar argumentation frameworks} (QBAFs) under gradual semantics~\cite{Baroni_18}, contribute attacks/supports between arguments, to interactively obtain BAFs as in Figure~\ref{fig:toy}  towards resolving conflicts on the agents' \emph{stances} on \cut{arguments acting as} explananda. 
We use QBAFs, which are BAFs where arguments are equipped with intrinsic strengths,
as they are well suited to modelling
private viewpoints, public conflicts, and resolutions, as well as cognitive \emph{biases}, \A{which are} important in XAI \cite{Bertrand_22}. We use gradual semantics  to capture individual \emph{evaluations} of stance, taking biases into account.
%\todo{``I felt the paper spends too much space on the infrastructure and relatively few words on why this infrastructure is the right one in the first place.'' We agree that better motivation will improve the paper; we will add to it. In a nutshell, it comes from XAI (e.g. Miller2019), and in particular from the idea that users have their own mental models representing their own knowledge/reasoning when engaging in explanations with models, which also host their own knowledge/reasoning. The use of argumentation is motivated by its versatility in capturing private viewpoints, public conflicts, and resolutions.} 
%\todo{``For most of the paper it was not so clear to me why the paper emphasized so much XAI (this is to some extent addressed by the agents' strategies and the discussion on explanatory properties, but I remain not fully convinced by the perspective taken in the paper – more on this point below).'' 
%We draw motivation from XAI for design choices, we will add discussion (see reply to Reviewer1)}

\item We identify and assess several %theoretical 
properties (§\ref{sec:xai}) which AXs may satisfy to be rendered suitable in an XAI setting. These properties concern, amongst others, the representation and possible resolution of conflicts within interactive explanations drawn from AXs.
%\todo{``Specifically, with so much focus put on XAI at the beginning of the paper, I would have expected properties of explanations (connectedness, acyclicity, and contributor irrelevance) to be presented explicitly much earlier, and discussed further. In terms of explanation, my intuition is that the core properties are resolution and conflict representation, which guarantee some 'trace' of the outcome.’’ We need to define AXs before properties, but we will mention the properties earlier in the paper as motivation.}

\item
We instantiate AXs %in the case of two-agent systems composed of
to the %traditional 
standard XAI setting of
two agents, a machine and a human, \cut{engaged in interactive explanations,} and define a catalogue of  agent behaviours for  \cut{the agents participating in these AXs in}
this setting
(§\ref{sec:agents}). We experiment in a simulated %setting
environment (§\ref{sec:evaluation}) with the behaviours, exploring \A{five}
%5 
hypotheses \cut{on relations between them by machine and human } %with regards to 
about conflict resolution and the accuracy of contributed arguments %in achieving 
towards it, noting that the strongest argument %may 
\A{is} not always %be 
the most effective. 
\end{itemize}
\cut{Finally, we conclude and look ahead to future work (§\ref{sec:conclusions}).}

%%%%%%%%%%%%%%%%%%%%%%%%%%%%%%%%%%%%%%%%%%%%%%%%%%%%%%%%%%%%%%%%%%%%%%%%%%%%%%%%%%%%%%%%%%%%%%%%%%%%

\section{Related Work}
\label{sec:related}

There is a vast literature on \emph{multi-agent argumentation}, 
e.g. recently, 
\cite{Raymond_20}
define an argumentation-based human-agent architecture integrating
regulatory compliance, suitable for human-agent path deconfliction and based on abstract argumentation~\cite{Dung_95}; 
\cite{Panisson_21} develop a multi-agent frameworks whereby agents can exchange information to jointly reason with argument schemes and critical questions; and 
 \cite{Tarle_22}
let agents debate using a shared abstract argumentation framework%~\cut{\cite{Dung_95}}
.
These works mostly focus on narrow settings using structured and abstract argumentation under extension-based semantics, and %, for the most part, do not consider 
mostly ignore the XAI angle (\cite{Raymond_20,Calegari_22} are exceptions). Instead, with XAI as our core drive, we focus on (quantitative) bipolar argumentation under gradual semantics,  motivated by their usefulness in several \cut{argumentation-based} XAI approaches (e.g. in \cite{Cocarascu_19,Albini_21,Potyka_21,Rago_22}). Other works consider (Q)BAFs in multi-agent argumentation, e.g. \cite{Kontarinis_15}, but not for XAI. We adapt some aspects of these works on multi-agent argumentation approaches, specifically the
idea of agents contributing attacks or supports (rather than arguments) to debates~\cite{Kontarinis_15} and the restriction to trees rooted at explananda under gradual semantics from \cite{Tarle_22}. We leave other interesting aspects %\AR{that} NOT NEEDED
they cover to future work, notably handling maliciousness~\cite{Kontarinis_15}, regulatory  compliance~\cite{Raymond_20}, and defining suitable utterances~\cite{Panisson_21}.

%\todo{argumentation for explanation? NOT SURE NEEDED \AR{agreed, we've mentioned the surveys so I think it's enough}}
Several approaches to obtain argumentative explanations for \cut{various} AI models exist (see \cite{Cyras_21,Vassiliades_22} for %recent 
overviews), often relying upon argumentative abstractions of the models\cut{ to be explained}. Our approach %here 
is orthogonal, %in that 
as we assume that suitable QBAF abstractions of models and humans % in terms of QBAFs 
exist, focusing instead on formalising and validating interactive explanations\cut{as argumentation}.

Our AXs and agent behaviours are designed to \emph{resolve conflicts} and are thus related to %other works aiming at 
works on conflict resolution, e.g. \cite{Black_11,Fan_12}, or centered around conflicts, e.g.~\cite{Pisano_22}, but these works have different purposes \A{to} interactive XAI and use forms of argumentation other than (Q)BAFs under gradual semantics.  
Our agent behaviours can also be seen as attempts at \emph{persuasion} in that they aim at selecting most efficacious arguments for changing the mind of the other agents, as %in computational persuasion (e.g. see
e.g. in \cite{Fan_12_MD,Hunter_18,Calegari_21,Donadello_22}.
\cut{Our focus on (Q)BAFs is in line  with work in this area, advocating the need for a support relation, as without it argumentation ``does not account for all of the positive relations between the statements viewed by the participants'' \cite{Hunter_18}. 
}
Further, our AXs can be seen as supporting forms of  \emph{information-seeking and inquiry}, %in that 
as they allow agents to share information, and are thus related to work in this spectrum (e.g. \cite{Black_07,Fan_15_PRIMA}).   
Our framework however differs from general-purpose forms of argumentation-based persuasion/information-seeking/inquiry in its focus on interactive XAI supported by (Q)BAFs under gradual semantics.

The importance of machine handling of \emph{information from humans} when explaining outputs, rather than the humans exclusively receiving information, has been highlighted e.g. %throughout various areas of AI, from users correcting 
for recommender systems \cite{Balog_19,Rago_20} %to %experts humans debugging classifiers
and debugging \cite{%Lucas17,
Lertvittayakumjorn_20} or other human-in-the-loop methods %\cite{Schramowski_20} 
(see \cite{Wu_22} for a survey).
%In these works, interaction amounts to  human feedback, whereas 
Differently from these works, we capture \emph{two-way} interactions.
%\todo{ contestability?   possibly mention ...} \todo{\cite{Miller_23X}?}

Some works advocate \emph{interactivity} in XAI~\cite{guilherme}, but do not make concrete suggestions on how to support it.
Other works advocate %the use of 
dialogues for XAI \cite{Lakkaraju_22X}, but it is unclear how these can be generated%to fully support interactivity
. We contribute to grounding the problem of generating interactive explanations by %giving 
a computational framework % and implementing it 
implemented 
in a simulated environment. 

%\todo{\cite{Wachsmuth_17}: comparison of argument evaluation...? NOT SURE NEEDED \AR{agreed, it'll just raise questions}}

\section{Preliminaries}
\label{sec:preliminaries}

%\todo{add some running example? preparing for Fig 1?} \AR{yes, I think we could push the depth of explanations. do we want to have it in the XAI setting?} \FT{yes, it would be nice} \AR{I think it needs to start in the next section}

A %bipolar argumentation framework (BAF) 
\A{BAF} \cite{Cayrol:05} is a triple
$\langle \Args, \Atts, \Supps \rangle$ such that $\Args$ is a finite set (whose elements are \emph{arguments}), $\Atts \subseteq \Args \times \Args$ (called the \emph{attack} relation) and $\Supps \subseteq \Args \times \Args$
(called the \emph{support} relation), where $\Atts$ and $\Supps$ are disjoint.
A %quantitative bipolar argumentation framework (QBAF) \todo{twicex2}
\A{QBAF} \cite{Baroni_15} is a quadruple
$\langle \Args, \Atts, \Supps, \BS \rangle$ such that $\langle \Args, \Atts, \Supps\rangle$ is a BAF and 
$\BS: \Args \rightarrow \Range$ ascribes \emph{base scores} to arguments; these  are values in some given $\Range$ representing the arguments' intrinsic strengths. 
Given BAF  
$\langle \Args, \Atts, \Supps\rangle $ or QBAF 
$\langle \Args, \Atts, \Supps, \BS \rangle $, for 
any $\arga \in \Args$, we call $\{ \argb \in \Args | (\argb, \arga) \in \Atts \}$  the \emph{attackers} of $\arga$ and $\{ \argb \in \Args | (\argb, \arga) \in \Supps \}$  the \emph{supporters} of $\arga$.

%In this paper, we 
We 
make use of the following notation:
%\begin{itemize}
%    \item 
given BAFs $\BAF\!=\!\langle \Args, \Atts, \Supps \rangle$, $\BAF'\!=\!\langle \Args', \Atts', \Supps' \rangle$, we say that $\BAF \sqsubseteq \BAF'$ \Iff\ $\Args\subseteq \Args'$,
$\Atts\subseteq \Atts'$ and
$\Supps\subseteq \Supps'$; 
also, we use $\BAF' \setminus \BAF$ to denote $\langle \Args' \setminus \Args, \Atts' \setminus \Atts, \Supps' \setminus \Supps \rangle$.
Similarly,
given QBAFs $\QBAF\!=\!\langle \Args, \Atts, \Supps, \BS \rangle $, $\QBAF'\!=\!\langle \Args', \Atts', \Supps', \BS' \rangle$, we say that $\QBAF \sqsubseteq \QBAF'$ \Iff\ $\Args\subseteq \Args'$,
$\Atts\subseteq \Atts'$, 
$\Supps\subseteq \Supps'$ and
$\forall \arga \in \Args\cap \Args'$ (which, by the other conditions, is exactly $\Args$), it holds that $\BS'(\arga)=\BS(\arga)$. Also, we use $\QBAF'\! \setminus \! \QBAF$ to denote $\langle \Args' \setminus \Args, \Atts' \setminus \Atts, \Supps' \setminus \Supps, \BS'' \rangle$, where $\BS''$ is $\BS'$ restricted to the arguments in $\Args' \setminus \Args$.\footnote{Note that $\BAF'\! \setminus \! \BAF$, $\QBAF' \! \setminus  \! \QBAF$ may not be BAFs, QBAFs, \resp, as  they may include    no arguments but non-empty attack/support relations.} %\todo{``we have $B'\setminus B$ defined in a way that seems to allow $B'\setminus B$ to have an empty set of arguments but a non-empty set of attacks/supports.'' You are right that the set of arguments in $B'\setminus B$ may be empty and thus $B'\setminus B$ may not amount to a BAF. We will add a comment but make no changes, as we need the current notion in Defs.5-6, in particular to accommodate the case where no arguments are learnt but relations are. } }}
%\end{itemize}
Given a BAF $\BAF$ and a QBAF $\QBAF=\langle \Args, \Atts, \Supps, \BS \rangle $, with an abuse of notation
we use $\BAF \sqsubseteq \QBAF$ to stand for
$\BAF \sqsubseteq \langle \Args, \Atts, \Supps \rangle$ and $\QBAF \sqsubseteq \BAF$ to stand for
$\langle \Args, \Atts, \Supps \rangle \sqsubseteq \BAF$.
For any BAFs or QBAFs $\mathcal{F}, \mathcal{F'}$,
we say that $\mathcal{F} \!=\! \mathcal{F}'$ \Iff\ $\mathcal{F} \!\sqsubseteq\! \mathcal{F}'$ and 
$\mathcal{F}' \!\sqsubseteq\! \mathcal{F}$, and $\mathcal{F} \!\sqsubset\! \mathcal{F}'$ \Iff\ $\mathcal{F} \!\sqsubseteq\! \mathcal{F}'$ but $\mathcal{F} \!\neq\! \mathcal{F}'$.
%We also deploy similar notations for $\cup$. \todo{AR to get rid - define in SM}

Both BAFs and QBAFs may be equipped with a \emph{gradual semantics} $\SF$, e.g. as in \cite{Baroni_17} for BAFs and as in \cite{Potyka_18} for QBAFs (see \cite{Baroni_19} for an overview), ascribing to arguments a \emph{dialectical strength} from within some given $\Range$ (which, in the case of QBAFs, is typically the same as for base scores): thus, for a given BAF or QBAF $\mathcal{F}$ and argument $\arga$, $\SF(\mathcal{F},\arga) \in \Range$. 
%
%\delete{In this paper, any agent $\AGa$ is equipped with a QBAF and a gradual semantics: the former provides an abstraction of the agent's %knowledge
%reasoning, with the base score representing  \emph{biases} over arguments, and the latter can be seen as an \emph{evaluation method} for arguments. 
%To reflect the use of QBAFs in our multi-agent explanatory setting, we adopt this terminology (of biases and evaluation methods) in the remainder. For illustration, in the setting of Figure~\ref{fig:toy}, %base scores/
%biases  may result from aggregations of votes from reviews for the machine and personal views for the human, and %gradual semantics/
%evaluation methods allow %to compute 
%the computation of the machine/human stance on the recommendation during the interaction %\delete{, as we shall see}
%(as in \cite{Cocarascu_19}).}

Inspired by~\cite{Tarle_22}'s use of (abstract) 
argumentation frameworks \cite{Dung_95} of a restricted  kind (amounting to  trees rooted with a single argument of focus), we use restricted BAFs and QBAFs% as follows.
:

\begin{definition}
\label{def:tree}
 Let $\mathcal{F}$ be a BAF $\langle \Args, \Atts, \Supps \rangle$ or
QBAF $\langle \Args, \Atts, \Supps, \BS \rangle$. 
For any arguments 
$\arga, \argb \in \Args$,
let a \emph{path} from $\arga$ to $\argb$ be defined as %$\argpaths: \Args \times \Args \rightarrow (\Atts \cup \Supps) \times \ldots \times (\Atts \cup \Supps)$ where for any  
$%\argpaths(\arga,\argb) = 
(\argc_0,\argc_1), \ldots, (\argc_{n-1}, \argc_{n})$ for some $n>0$ (referred to as the \emph{length} of the path) where $\argc_0 = \arga$, $\argc_n = \argb$ and, for any $1 \leq i \leq n%\in \{1, \ldots, n \}
$, $(\argc_{i-1}, \argc_{i}) \in \Atts \cup \Supps$.\footnote{Later, we  
%Then, the distance between $\arga$ and $\argb$ is $n$. %$\argdist: \Args \times \Args \rightarrow \mathbb{Z}^+$ where for any $\arga, \argb \in \Args$, $\argdist(\arga,\argb) = | minpath |$ where $minpath$ is the shortest path between $\arga$ and $\argb$.
will use $\argpaths(\arga,\argb)$ to indicate the set of all paths between arguments $\arga$ and $\argb$, leaving the (Q)BAF implicit, and  use $|p|$ %to indicate 
for the length of path $p$. Also, we may see paths as sets of pairs.} 
Then, 
for $\arge \in \Args$,
$\mathcal{F}$ is 
a \emph{BAF/QBAF} (\resp)  \emph{for $\arge$} \Iff\
%\begin{itemize}
    %\item 
    i) $\nexists (\arge,\arga) \in \Atts \cup \Supps$;
    ii) $\forall \arga \in \Args \setminus \{\arge\}$, %$\exists \arga_1, \ldots, \arga_n \in \Args$ for $n > 1$ such that $\arga=\arga_1$, $\arga_n=\arge$, and $\forall %i \in \{1, \ldots, n-1\}
    %1 \leq i \leq n-1$, $(\arga_i, \arga_{i+1}) \in \Atts \cup \Supps$, and
    there is a path from $\arga$ to $\arge$; and 
    iii) %$\nexists \arga_1, \ldots, \arga_n \in \Args$ for $n > 1$ such that $\arga_1=\arga_n$ and $\forall %i \in \{1, \ldots, n-1\}
    %1 \leq i \leq n-1$, $(\arga_i, \arga_{i+1}) \in \Atts \cup \Supps$. 
$\nexists \arga \in \Args$  with a path from $\arga$ to $\arga$.
%\end{itemize}
%$\arge \in \Args$ is the root of the graph formed by the nodes of $\Args$ and the edges of $\Atts \cup \Supps$ where all attacks and supports are directed toward the root $\arge$, i.e. for all $\arga, \argb \in \Args$, if there exists a path from $\arga$ to $\argb$, then that path is a subset of a path from $\arga$ to $\arge$. \todo{make more formal?} \todo{use \cite{Fan_15}}
\end{definition}
Here $\arge$ plays the role of an \emph{explanandum}.\footnote{Other terms to denote the ``focal point'' of BAFs/QBAFs could be used. We use \emph{explanandum} given our focus on the XAI setting.} When interpreting the BAF/QBAF as  a graph (with arguments as nodes and attacks/supports as edges), %the first bullet 
i) amounts to sanctioning that $\arge$ admits no outgoing edges, %the second bullet 
ii) that %there is a path from any  node to $\arge$
$\arge$ is reachable from any other node, and %the third bullet 
iii) that there are no cycles in the graph (and thus, when combining the three %bullets
requirements, the graph is a multi-tree rooted at $\arge$). %\todo{not really, as there could be multiple paths....: multi-tree? set of trees?}
The restrictions in Definition~\ref{def:tree} impose that every argument in a BAF/QBAF for \cut{an argument}
$\arge$ are ``related'' to $\arge$, in the spirit of \cite{Fan_15}. %\delete{(who, however, focus on abstract argumentation and \emph{ABA frameworks},  rather than BAFs/QBAFs, and extension-based, rather than gradual, semantics)}. 

In all illustrations %in the paper 
(and in some of the experiments in §\ref{sec:evaluation}) we use the \emph{DF-QuAD} gradual semantics  \cite{Rago_16} for QBAFs for explananda%(see \cite{Baroni_19} for an application to QBAFs)
. This uses $\mathbb{I}=[0,1]$ and: %relies upon
\begin{itemize}
    \item 
a \emph{strength aggregation function} $\Sigma$ such that
$\Sigma(())\!=\!0$ and, for $v_1,\ldots,v_n \!\in \![0,1]$ ($n \geq 1$),
if $n=1$ then $\Sigma((v_1))=v_1$,
if $n=2$ then $\Sigma((v_1,v_2))=%v_1+(1-v_1)\cdot v_2 = 
v_1 + v_2 - v_1\cdot v_2$,
and if $n>2$ then
$\Sigma((v_1,\ldots,v_n)) = \Sigma (\Sigma((v_1,\ldots, v_{n-1})),v_n)$;
% : %\mathbb{I} [0,1]^* \rightarrow %\mathbb{I} [0,1]$, where for $\mathcal{S} = (v_1,\ldots,v_n) \in %\mathbb{I} [0,1]^*$: if $n=0$, $\Sigma(S) = 0$; %\\
%if $n=1$, $\Sigma(S) = v_1$; %\\
%if $n=2$, $\Sigma(S) = f(v_1, v_2)$; %\\
%if $n>2$, $\Sigma(S) = f(\Sigma(v_1,\ldots,v_{n-1}), v_n)$; %\\
%\end{gather*} 
%and with \emph{base function} $f:%\mathbb{I}\times\mathbb{I}\rightarrow\mathbb{I}
%[0,1]\times [0,1] \rightarrow [0,1]$ defined as:
%\begin{gather*} $f(v_1,v_2)=v_1+(1-v_1)\cdot v_2 = v_1 + v_2 - v_1\cdot v_2$ for $v_1, v_2\in%\mathbb{I} [0,1]$.
%\end{gather*}
\item
a \emph{combination function} $c% :[0,1]\times [0,1]\times  [0,1]\rightarrow [0,1]
$ %combines $v^-$ and $v^+$ with the argument's base score ($v^0$) 
such that, for $v^0,v^-,v^+ \in [0,1]$:
%\begin{gather*} 
if $v^-\geq v^+$ then $c(v^0,v^-,v^+)=v^0-v^0\cdot\mid v^+ - v^-\mid$ and
if $v^-< v^+$, then
$c(v^0,v^-,v^+)=v^0+(1-v^0)\cdot\mid v^+ - v^-\mid$.
\end{itemize}
Then, for $\mathcal{F}\!=\!\langle \Args, \Atts, \Supps, \BS \rangle$ and any 
$\arga \!\in \!\Args$, given
$\Atts(\arga) \!=\! \{ \argb \!\in\! \Args | (\argb, \arga) \!\in\! \Atts \}$ and
$\Supps(\arga) \!=\! \{ \argb \!\in\! \Args | (\argb, \arga) \!\in\! \Supps \}$,
$\SF(\mathcal{F},\arga)=c(\BS(\arga),\Sigma(\SF(\mathcal{F},
\Atts(\arga))),\Sigma(\SF(\mathcal{F},
\Supps(\arga))))$
%\end{gather*} 
where, for any  $S \!\subseteq\! \Args$%of arguments \HL{in a framework $\mathcal{F}$} FT: we have already given \mathcal{F} up front
,  $\SF(\mathcal{F},S)\!=\!(\SF(\mathcal{F},\arga_1),\ldots,\SF(\mathcal{F},\arga_k))$ for $(\arga_1,\ldots,\arga_k)$, an arbitrary permutation of %\cut{the elements in} 
$S$. %\todo{this needs how $\SF$ is used on a set} DONE?
%\todo{good spot but the framework needs to be first as I've corrected it. also I think when the strength takes a set of arguments it needed my addition in turquoise as it's a bit messy otherwise?}

\section{Argumentative Exchanges (AXs)}
\label{sec:exchanges}

%\todo{``What is the dialectical strength of an argument? What does it mean to have "biases" over arguments?’’ Basic descriptions of dialectical strength and bias are in Section3. We will add intuitions for interactive explanations in Sections 3 and 4 (in a nutshell, dialectical strength gives a measure of agents’ stances and biases represent their view on the “quality” of arguments before other arguments are considered). }

We define AXs as a general framework in which \emph{agents} %to 
argue
%exchange arguments 
with the goal of conflict resolution.  
%\delete{We use desiderata of argumentative agent protocols, introduced by \cite{McBurney_02} for domains such as persuasion, negotiation and inquiry, as guiding principles, some of which we will elaborate on.}
%\todo{I would avoid vague statements such as this:  which ones do we not consider?  are we really using them? very vague mentions later...either we give them as proposetrties or not say much really} 
%\delete{While the examples here and the sections which follow are placed in the specific setting of XAI, the main content of this section is deliberately kept general, since we believe (as future work) our approach could be applied in other domains such as those mentioned.}
The conflicts may arise when agents  hold different \emph{stances} on explananda. %\delete{(in line with \emph{diversity of individual purposes} \cite{McBurney_02})}
To model these settings, we rely upon QBAFs for explananda as abstractions of agents' %inner knowledge and workings
internals%\delete{; then agents' stances are determined by their own biases ($\BS$ in their  QBAFs) and evaluation methods (%using each agent's own $\SF$)}
. 
%
%We assume that 
Specifically, we assume that each agent $\AGa$ is equipped with a QBAF and a gradual semantics ($\SF$): the former provides an abstraction of the agent's knowledge/reasoning, with the base score ($\BS$) representing  \emph{biases} over arguments; the latter can be seen as an \emph{evaluation method} for arguments. 
To reflect the use of QBAFs in our multi-agent explanatory setting, we adopt this terminology (of biases and evaluation methods) in the remainder.
Intuitively, biases and evaluations  represent agents' views on the quality of arguments before and after, \resp, other arguments are considered.
For illustration, in the setting of Figure~\ref{fig:toy}, %base scores/
biases  may result from aggregations of votes from reviews for the machine and from personal views for the human, and %gradual semantics/
evaluation methods allow %to compute 
the computation of the machine/human stance on the recommendation during the interaction %\delete{, as we shall see}
(as in \cite{Cocarascu_19}).
Agents may choose their  own \emph{evaluation range} for measuring  biases/evaluating  arguments.

\begin{definition}
An \emph{evaluation range} $\Range$ is  a set equipped with a pre-order $\leq$ (where, as usual $x < y$ denotes $x \leq y$ and $y \nleq x$) such that $\Range = \RPos \cup \RNeu \cup \RNeg$
where $\RPos$, $\RNeu$ and $\RNeg$ are disjoint and for any $i \in \RPos$, $j \in \RNeu$ and $k \in \RNeg$, $k < j < i$. 
%\todo{what about elements within EACH OF THE THREE SETS? could they be ordered? incomparable?}
We refer to 
$\RPos$, $\RNeu$ and $\RNeg$,  \resp, as \emph{positive}, \emph{neutral} and \emph{negative evaluations}.
\end{definition}
Thus, an evaluation range discretises the space of possible evaluations into three categories.\footnote{We choose three discrete values only for simplicity.
This may mean \cut{, somewhat counter-intuitively,} that very close values, e.g. 0.49 and 0.51, belong to different categories. 
%\A{noting that this is an imperfect way of partitioning the evaluation ranges. We leave}\delete{leaving} 
We leave to future work the analysis of further value categorisations, e.g. a distinction between strongly and mildly positive values or %the 
\emph{comfort zones} %of an agent, as in 
\cite{Tarle_22}.  
%\todo{``Agents positions/stances are mapped to a 3-level qualitative scale: positive, neutral, or negative. An argumentative exchange (AX) is said to be resolved when agents reach the same qualitative level. But note that it means that an AX can be resolved with agents holding opinions (in terms of their values) much further away than in some unresolved case (eg. 0.49/0.51 vs. 0.01/0.49). This means that there is a compelling semantics attached to these values.’’ We agree and will add your point to footnote3. }
}

\begin{definition}
\label{def:agent}
A \emph{private triple} for an agent $\AGa$ and an explanandum $\arge$
is $(\RangeA,\QBAFa,\SFa)$ where: 
    \begin{itemize}
        \item  $\RangeA = \RPosA \cup \RNegA \cup \RNeuA$ %, referred to as 
        is an evaluation range, referred to as $\AGa$'s \emph{private 
        evaluation range};
        \item  $\QBAFa = \langle \ArgsA, \AttsA, \SuppsA, \BSa\rangle$ is a QBAF for $\arge$, referred to as $\AGa$'s \emph{private QBAF}, such that $\forall \arga \in \ArgsA$,  $\BSa(\arga) \in \RangeA$;
        \item  $\SFa$  is an evaluation method, referred to as $\AGa$'s \emph{private 
        evaluation method}, such that, for any QBAF $\QBAF \!=\! \langle \Args, \Atts, \Supps, \BS \rangle$ (%with
        $\BS:\! \Args\!\! \rightarrow \! \RangeA$) and,  for any $\arga\!\in \!\Args$,
         $\SFa(\QBAF,\arga) \!\in \!\RangeA$.
         %\todo{would this be more convenient/tidy if it was relation $\rightarrow t \times$ agent? FT: unclear what you mean}
    \end{itemize}
\end{definition}

Agents' stances on explananda %result from the evaluation thereof using the agents' private evaluation methods. Thus, they 
are determined by their private biases %($\BS$ in their  QBAFs) 
and evaluation methods% (using each agent's own $\SF$)
.

\begin{definition}
\label{def:stance}
Let  $(\RangeA,\QBAFa,\SFa)$ be  a private triple for agent  $\AGa$  (for some $\arge$),  with $\QBAFa = \langle \ArgsA, \AttsA, \SuppsA, \BSa\rangle$. Then, for $\arga \in \ArgsA$,
$\AGa$'s \emph{stance on $\arga$} is defined, for $* \in \{ -, 0, + \}$%\todo{where $-<0<+$}
, 
as $\StanceA(\QBAFa,\arga) = *$ iff $\SFa(\QBAFa, \arga) \in \RangeA^*$.
\end{definition}
Note that $\arga$ may be the explanandum  %$\arge$ 
or any other argument (namely, an agent may hold a stance on any arguments in its private QBAF). Also, abusing notation, we will lift the pre-order over elements of $\mathbb{I}$ to stances, whereby $-<0<+$.

In general,  agents %in a given \emph{multi-agent system} 
may hold different evaluation ranges, biases, QBAFs and evaluation methods, but the 
 discretisation of the agents' evaluation ranges to obtain their stances allows for direct comparison across agents.

\begin{figure}[t]
\centering
    \includegraphics[width=1.04\linewidth]{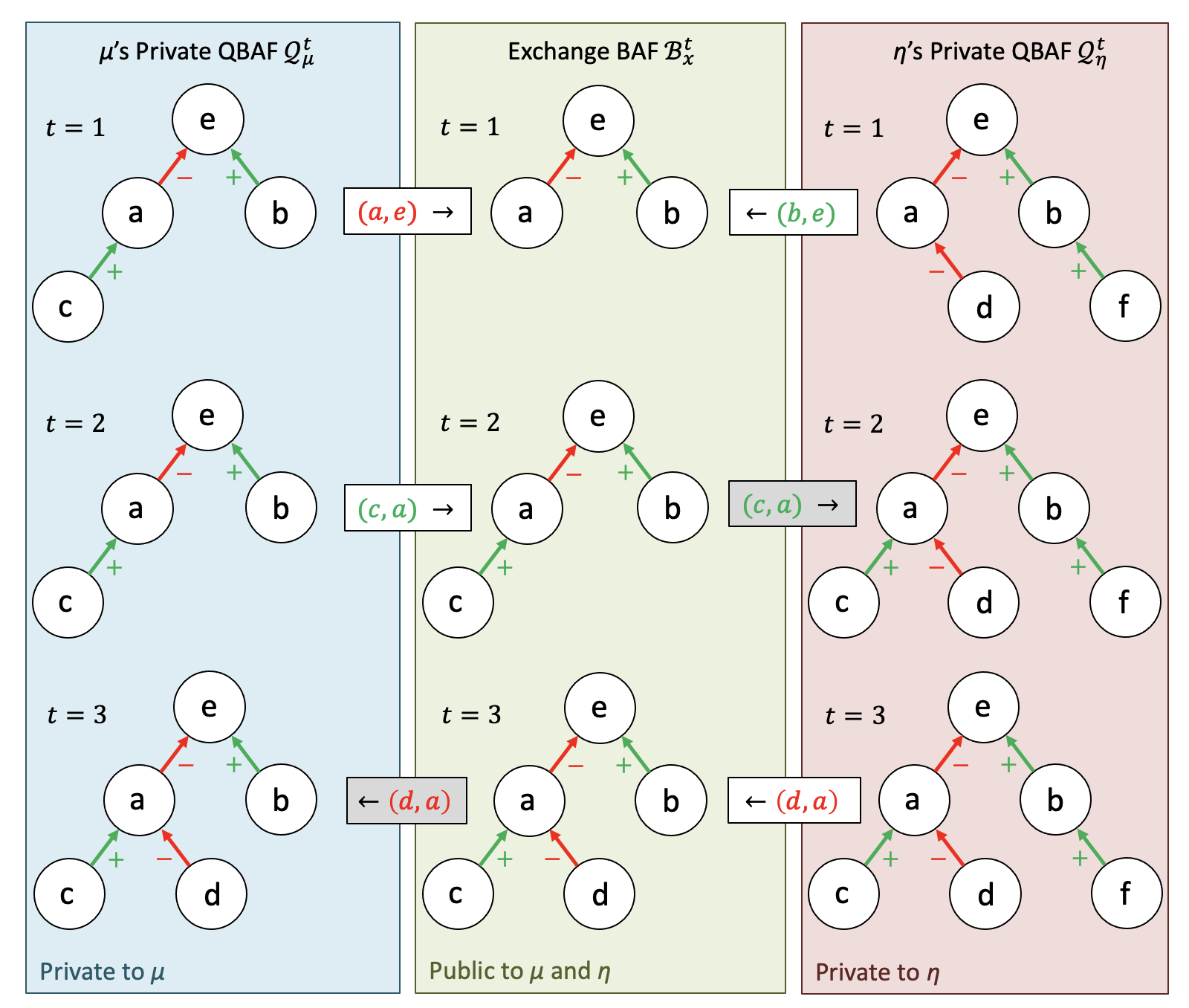}
    \protect\caption{%A graphical view of an 
    AX for %\AR{the} 
    explanandum $\arge$ amongst agents $\Agents=\{\AGm, \AGh\}$, with the exchange BAF representing an interactive explanation. White (grey) boxes represent contributions (learnt relations, \resp).
    %\todo{``From the title, I expected to have a clear idea of what could be an explanation in an interactive setting, but I may have missed something.  For instance, what is considered an explanation in Figure 2 at time 3? The common public tree? Or the private one for the human agent?’’ Your intuition about the “common public tree” as explanations for humans is correct. We will emphasise this, specifically for Figs1 and 3.}
    }
    \label{fig:structure}
\end{figure}

\begin{example}\label{ex:agents}
Consider %two agents $\AGa$ and $\AGb$
%a machine agent $\AGm$ explaining the movie review $\arge$ to a human agent $\AGh$,
a machine agent $\AGm$ and a human agent $\AGh$
equipped \resp\ with private triples $(\RangeM,\QBAFm,\SFm)$ and $(\RangeH,\QBAFh,\SFh)$, with $\QBAFm\!=\!\langle \ArgsM, \AttsM, \SuppsM, \BSm\rangle$, $\QBAFh=\langle \ArgsH, \AttsH, \SuppsH, \BSh \rangle$
QBAFs for the same %explanandum 
$\arge$  and: 
\begin{itemize}
    \item $\RNegM = \RNegH = [0,0.5)$, $\RNeuM =\RNeuH = \{0.5\}$ and $\RPosM = \RPosH = (0.5,1]$; 
    %
    %\item $\RNegA = [-1,0)$, $\RNeuA = \{0\}$ and $\RPosA = (0,1]$; 
    %
    %\item $\RNegB = [0,0.5)$, $\RNeuB = \{0.5\}$ and $\RPosB = (0.5,1]$;
    %
    \item $\ArgsM = \{ \arge, \arga, \argb, \argc \}$, 
    $\AttsM = \{(\arga, \arge)\}$, 
    $\SuppsM = \{ (\argb, \arge), (\argc, \arga)\}$ (represented graphically on the top left 
    of Figure \ref{fig:structure}) and 
    $\BSm(\arge) = 0.7$,
    $\BSm(\arga) = 0.8$,
    $\BSm(\argb) = 0.4$, and
    $\BSm(\argc) = 0.6$;
    \item $\ArgsH \!\!=\!\! \{ \arge, \arga, \argb, \argd, \argf \}$,$\AttsH \!\!=\!\! \{(\arga, \arge), (\argd, \arga)\}$,$\SuppsH \!\!=\!\! \{(\argb, \arge), (\argf, \argb)\}$ (represented %graphically 
    on the top right 
    of Figure \ref{fig:structure}) and 
    $\BSh(\arge) = 0.6$,
    $\BSh(\arga) = 0.8$,
    $\BSh(\argb) = 0.2$,
    $\BSh(\argd) = 0.6$ and
    $\BSh(\argf) = 0.5$.
    %
    %\item $\SFa$ is the max-based semantics in \cite{Amgoud_08}, giving 
    %$\SFa(\QBAFa,\arge) = -0.25$,
    %$\SFa(\QBAFa,\arga) = 0.5$,
    %$\SFa(\QBAFa,\argb) = 0$, and
    %$\SFa(\QBAFa,\argc) = 0$; 
    %
    \item $\SFm$ is the DF-QuAD semantics, giving 
    $\SFm(\QBAFm,\arge) = 0.336$,
    $\SFm(\QBAFm,\arga) \!=\! 0.92$,
    $\SFm(\QBAFm,\argb) \!=\! 0.4$, and
    $\SFm(\QBAFm,\argc) \!=\! 0.6$; 
    \item $\SFh$ is also %the DF-QuAD semantics, 
    %in \cite{Rago_16}, 
     DF-QuAD% too
     , giving 
    $\SFh\!(\!\QBAFh,\arge\!) \!\!=\!\! 0.712$,$\SFh(\QBAFh,\arga) \!=\! 0.32$,
    $\SFh(\QBAFh,\argb) \!=\! 0.6$,
    $\SFh(\QBAFh,\argd) \!=\! 0.6$,
    $\SFh(\QBAFh,\argf) \!=\! 0.5$.
\end{itemize}
%\AR{Here, $\argc$ and $\argd$ may represent different actors performing poorly or well, \resp, while $\argf$ may represent a certain scene being shot beautifully.}
Thus, the machine and human agents hold entirely different views on the arguments (based on their private %\delete{evaluation ranges, QBAFs and evaluation methods} 
QBAFs and their evaluations% of arguments
)
and 
$\StanceM(\QBAFm,\arge) = -$ while $\StanceH(\QBAFh,\arge) = +$. Thus, there is a conflict between the agents' stances 
 on $\arge$.
\end{example}
We define AXs so that they can provide the ground to identify and resolve 
%\AR{facilitate the identification and resolution of} 
conflicts in stance amongst agents.

 \begin{definition}
\label{def:exchange}
An \emph{Argumentative eXchange (AX) for an explanandum $\arge$ amongst %$k$ 
agents} %\delete{$\AGa(1), \ldots \AGa(k)$ (for $k \geq 2$)} 
$\Agents$ (where $|\Agents| \geq 2$)
is a tuple
$\langle %\delete{\BAFu,}   
\BAFi^0, \ldots, \BAFi^{n},   
\Agents^0, \ldots,\Agents^n, 
\SpeakerM \rangle$ where $n>0$ and:
%\AR{$$\langle \BAFu,    \BAFi^0, \ldots, \BAFi^{n},    \Agents^0, \ldots,\Agents^n, \SpeakerM \rangle$$}
%begins at timestep $t=0$ with
%where $n > 0$ and:

\begin{itemize}
    %\item \delete{$\BAFu$=$\langle \ArgsU, \AttsU, \SuppsU \rangle$ is a BAF for $\arge$, called the \emph{\universal\ BAF};}
    \item for every timestep $0 \leq t \leq n$:
     \begin{itemize}
\item %for every timestep $0 \leq t \leq n$, 
     $\BAFi^t = \langle \ArgsI^t, \AttsI^t, \SuppsI^t \rangle$ is a BAF for $\arge$, called the \emph{\public\ 
      BAF at $t$},  such that
     %\delete{$\BAFi^t \sqsubseteq \BAFu$
     %i.e. $\forall \arga, \argb \in \ArgsI^t$, if $(\arga, \argb) \in \AttsI^t$ then $(\arga, \argb) \in \Atts_u$ and if $(\arga, \argb) \in \SuppsI^t$ then $(\arga, \argb) \in \Supps_u$
     %; also,} 
     $\ArgsI^0 = \{ \arge \}$, $\AttsI^0 = \SuppsI^0 = \emptyset$ and for
     $t > 0$,
     %$\ArgsI^t \supseteq \ArgsI^{t-1}$, $\AttsI^t \supseteq \AttsI^{t-1}$ and $\SuppsI^t \supseteq \SuppsI^{t-1}$
     $\BAFi^{t-1} %\sqsubseteq 
     %\FT{\sqsubset} \BAFi^{t}$; \todo{NEEDS CHANGING FOR PASS?}
     \sqsubseteq \BAFi^{t}$;
    
    \item %for every timestep $0 \leq t \leq n$, 
    $\Agents^t$ is a set of 
    %\delete{$k$} 
    %elements
    private triples $(\RangeA^t,\QBAFa^t,\SFa^t)$ for $\arge$, one for each  agent %\delete{$\AGa \in \{\AGa(1), \ldots, \AGa(k)\}$} 
    $\AGa \in \Agents$,  
    where%$\QBAFa^t=\langle \ArgsA^t, \AttsA^t, \SuppsA^t, \BSa^t \rangle$,   $\langle \ArgsA^{t}, \AttsA^{t}, \SuppsA^{t} \rangle \sqsubseteq \BAFu$, 
    , for $t >0$, $\RangeA^{t-1}=\RangeA^t$, $\SFa^{t-1}=\SFa^t$,
                $\QBAFa^{t-1} \sqsubseteq \QBAFa^{t}$ and
                 %($\ArgsA^t\setminus \ArgsA^{t-1}) \subseteq (\ArgsI^t \setminus \ArgsI^{t-1})$, $(\AttsA^{t}\setminus\AttsA^{t-1}) = (\AttsI^t \setminus \AttsI^{t-1}) \cap (\ArgsA^t \times \ArgsA^t$) and  $(\SuppsA^{t}\setminus\SuppsA^{t-1}) = (\SuppsI^t \setminus \SuppsI^{t-1}) \cap (\ArgsA^t \times \ArgsA^t$); 
                $\QBAFa^{t} \setminus \QBAFa^{t-1} \sqsubseteq 
                \BAFi^t \setminus \BAFi^{t-1}$; 
               % \AR{$\BAFi^t \setminus \BAFi^{t-1} \sqsubseteq \QBAFa^{t}$}; \todo{FT is not sure about deleting the first bit: we need to say where the learnt arguments are coming from - namely the new bits in the exchange - else agents could completely invent new stuff; the added condition is too strong in my mind, and should be a choice of the agents -- in the policy, not in this definition (this is a definition that looks at what happens, and enforces behaviour minimally} \AR{okay I'm happy to enforce something like what I added in later sections}

     \end{itemize}
     
    \item $\SpeakerM$, referred to as the \emph{contributor mapping}, is a mapping such that, for every
    $(\arga,\argb) \in \AttsI^{n} \cup \SuppsI^{n}$: $\SpeakerM((\arga,\argb))=(\AGa,t)$ with $0 < t \leq n$ and %\delete{$\AGa \in \{\AGa(1), \ldots, \AGa(k)\}$} 
    $\AGa \in \Agents$. %\todo{I think we need to say that the contributions come from the private QBAFs? NO - CHOICE - ADD LATER also that they can't be repeated? NO, BECAUSE MAPPING - SO THEY CANNOT BE REPEATED }
    \end{itemize}
    \end{definition}
 %
     %\delete{Intuitively, the \emph{\universal\ BAF} serves as a lingua franca amongst agents: we enforce that all arguments, attacks and supports of the \emph{\public\ BAF} and of the agents' \emph{private QBAF} and \emph{private view of \public\ QBAF} are drawn from the \universal\ BAF , so that they agree on whether arguments attack/support arguments,  once they are presented with them.} 
    %\delete{Argumentative exchanges are designed to treat all agents equally (in line with \emph{fairness} \cite{McBurney_02}), with}
    
    Agents' private  triples thus %capture their private knowledge and biases. These 
    change over time during AXs, with
    %We impose 
    several restrictions% on agents during argumentative exchanges
    , in particular that  %they 
    agents do not change their evaluation ranges and methods, and that their biases on known arguments propagate across timesteps (but note that Definition~\ref{def:exchange} does not impose any restriction on the agents' private triples at timestep 0, other than they are all for $\arge$). 
    %\AR{Note also that agents can only contribute attacks or supports which are new to the exchange and have not already been contributed. \todo{this is not the case} \todo{I think we wrote said this because the contributor mapping has one individual timestep?}}
    %\todo{WE DO NOT - HERE at least - MAY BE LATER We also impose that any attack/support  which is added to the \public\ BAF by an agent is \emph{learnt} by the other agents, i.e. the argument is added to the other agents private QBAFs (though its status and thus its impact therein depends on the agents).}
The restriction that all BAFs/QBAFs in exchanges are \emph{for the explanandum}, means that all contributed attacks and supports (and underlying arguments) are ``relevant'' to the explanandum. 
Implicitly, while we do not assume that agents share arguments, we %also
assume  that %agents 
they
agree on %the argumentative structure as a 
an underpinning `lingua franca', so that, in particular, if two agents are both aware of two arguments, they must agree on any attack or support between them, e.g. it cannot be that an argument attacks another argument for one agent but not for another (%this assumption  is 
in line with other works, e.g. %Assumption 2 in \cite{Tarle_22},  cultures in \cite{Raymond_20} and consensus over attacks in \cite{Gao_16}
\cite{Tarle_22,Raymond_20%,Gao_16
}).
We leave to future work the study of the impact of this assumption in practice when AXs take place between machines and humans.
%\delete{\A{Note that agents may implicitly express disagreement on contributed attacks or supports via their biases on the corresponding attacking or supporting, \resp, arguments, which we will discuss shortly.}}
%\todo{`` In the setting, do we consider that agents share the same set of arguments?  In other words, is $a \in X_\mu$ the same as $a \in X_\eta$ in Figure 2?  which seems a big assumption in an interaction between a human agent and an artificial agent. ’’ We do not assume that “agents share the same sets of arguments”, but we do assume, in line with, e.g., (de Tarlé Et Al 2022 and Raymond Et Al. 2020), that agents have a lingua franca (see discussion in the paragraph after Def5). We will expand that discussion, mentioning limitations and future work. }
    
    During AXs,  
    agents contribute elements of the attack/support relations, thus ``arguing'' with one another.  These elements cannot be withdrawn once contributed, in line with human practices, and, by definition of $\SpeakerM$, each element is said once by exactly one agent, thus avoiding repetitions that may occur in human exchanges. 
    %\delete{(in line with \emph{discouragement of disruption} \cite{McBurney_02})}. 
    %\FT{don't we need something like this speaker mapping to say who said what? else how do we go about analysing the impact of what agents said? Note that it needs to be defined on the relations, not the arguments, as the same argument may be in relations with different things and we do not want to force agents to say everything at once...} \AR{I'm not sure we need to know who said what? But yes maybe}
    Note that we do not require that all agents contribute something to an AX, namely
    it may be that $\{\AGa | \SpeakerM(%n,
    (a,b))=(\AGa,t), (a,b) \in \AttsI^n\cup \SuppsI^n\} \subset \Agents$% \delete{$\{\AGa(1), \ldots, \AGa(k)\}$}
    . %i.e. agents may keep quiet during an explanatory exchange
    %; however, necessarily $C \neq \emptyset$, since $n >0$ and $\BAFi^{t-1} \sqsubset \BAFi^{t}$ for every timestep $0 < t \leq n$   \FT{IF WE REPLACE $\BAFi^{t-1}     \sqsubset \BAFi^{t}$ WITH $\BAFi^{t-1}      \sqsubseteq \BAFi^{t}$ THIS COMMENT NEEDS TO BE CHANGED, as $C$ may be empty i.e. WE COULD HAVE THAT NOBODY CONTRIBUTES ANYTHING}. Also, by virtue of the latter restriction, at each timestep  other than the last some agents contribute some    attack or support to the exchange. \todo{FT: THIS MEANS THAT AGENTS CANNOT PASS; IF WE WANT TO ALLOW PASSES, THEN WE NEED TO REPLACE $\BAFi^{t-1} \sqsubset \BAFi^{t}$ WITH $\BAFi^{t-1}     \sqsubseteq \BAFi^{t}$}.
     Also, we do not force agents to contribute something at every timestep (i.e. it may be the case that $\BAFi^{t-1} =\BAFi^{t}$ at some timestep $t$). Further, while the definition of AX does not impose that agents are truthful%, e.g. they could contribute attacks or supports not in their private QBAFs. However,} 
     , from now on we will focus on truthful agents only and thus assume that if $(\arga,\argb) \in \AttsI^{n}$ or $\SuppsI^{n}$ and $\SpeakerM((\arga,\argb))=(\AGa,t)$ (with $0 < t \leq n$% and $\AGa \in \Agents$
     ),  then, \resp, $(\arga,\argb) \in \AttsA^{t-1}$ or $\SuppsA^{t-1}$%,  (thus when an agent shares an attack or support it is so in its private QBAF at the time)
     .

In the remainder, 
 we %will often 
 may denote the private triple $(\RangeA^t,\QBAFa^t,\SFa^t)$ %of an agent $\AGa$ %at timestep $t$ %simply 
 as $\AGa^t$ and the stance $\StanceA(\QBAFa^t,%\arge
 \arga)$  %simply 
 as $\StanceA^t(%\arge
 \arga)$.

\begin{example}\label{ex:exchange} 
%\FT{(Example~\ref{ex:agents} continued)}
    An AX 
    %for %the explanandum 
    %$\arge$ 
    amongst %the two agents %from Example \ref{ex:agents} 
    $\{\AGm,\AGh\}$ from Example~\ref{ex:agents}  may be  $\langle %\BAFu, 
    \BAFi^0, \BAFi^{1}, \Agents^0, \Agents^1, \SpeakerM \rangle$  such that (%partially visualised in the 
    see top row of Figure~\ref{fig:structure}):
    \begin{itemize}
        %\item \delete{$\BAFu \!\!=\!\! \langle \{ \arge, \arga, \argb, \argc, \argd, \argf \}, \{(\arga,\arge), (\argd,\arga) \}, \{ (\argb,\arge), (\argc,\arga), (\argf,\argb) \} \rangle$;}
        %
        \item $\BAFi^0=\langle \{\arge\},\emptyset,\emptyset\rangle$,  $\BAFi^1 = \langle \{ \arge, \arga, \argb\}, \{(\arga,\arge) \}, \{(\argb,\arge) \} \rangle$;
        %$\BAFi^2 = \langle \{ \arge, \arga, \argb \}, \{(\arga,\arge) \}, \{ (\argb,\arge) \} \rangle$ 
        \item $\AGm^0 = \AGm^1 %=\AGa^2
        $ and $\AGh^0 = \AGh^1 %=\AGa^2
        $ are as in Example~\ref{ex:agents};
        \item 
      $\SpeakerM((\arga, \arge)) = (\AGm,1)$ and $\SpeakerM((\argb, \arge)) = (\AGh,1)$, 
        i.e. $\AGm$ and $\AGh$ contribute, \resp, attack $(\arga,\arge)$  and support $(\argb,\arge)$ %in $\BAFi^1$ 
        at \cut{timestep} 
        1. 
                \end{itemize}
       Here, each agent contributes a single attack or support justifying their stances (negative for $\AGm$ and positive for $\AGh$), but, in general, 
       multiple agents may contribute multiple relations at single timesteps, or no relations at all.
           \end{example}

%\delete{Note that here and in the examples which follow, agents contribute a single attack/support at each timestep: in general, they could contribute multiple attacks/supports at the same timestep, e.g. $\AGa$ could have contributed both attack $(\arga,\arge)$  and support $(\argc,\arga)$ in one go.  Also, in general, there may be timesteps at which nothing is contributed, e.g. it may be the case that $\BAFi^1=\BAFi^0$. Finally, in the example only one agent contributes to the exchange BAF at timestep 1: in general, they may contribute attacks/supports in parallel, at the same timestep.}

%\delete{, though turn-making functions may be designed to constrain this freedom depending on the application of the explanatory exchange}. %\todo{FT: turn-making does not belong here in my mind - it is - again - about behaviour - part of agents. If mentioned at all we should say: we will see that turn-making can impose....possibly}

%\delete{As a result of argumentative exchanges,} 
When contributed attacks/supports 
%\FT{relations} 
are new to agents,
%agents 
they may 
\emph{(rote) learn} them, %attacks/supports
with the arguments %thereby introduced
they introduce.

\begin{definition}
Let $\langle %\BAFu, 
\BAFi^0, \ldots, \BAFi^{n},  \Agents^0,\ldots,\Agents^n, \SpeakerM \rangle$ be an AX  amongst agents 
$\Agents$. Then, for any  $\AGa \in \Agents$, with private tuples $(\RangeA^0,\QBAFa^0,\SFa^0)$, \ldots, $(\RangeA^n,\QBAFa^n,\SFa^n)$:
\begin{itemize}
    \item for any $0 < t \leq n$, for
    $\langle \ArgsA, \AttsA, \SuppsA, \BSa\rangle = \QBAFa^t \setminus \QBAFa^{t-1}$,  
    %\\
    $\ArgsA, \AttsA$, and $\SuppsA$ are, \resp, the \emph{learnt arguments, attacks, and supports} by $\AGa$ \emph{at timestep} $t$;
    \item for
    $\langle \ArgsA, \AttsA, \SuppsA, \BSa\rangle = \QBAFa^n \setminus \QBAFa^{0}$, 
    %\\ 
    $\ArgsA, \AttsA, $ and $\SuppsA$ are, \resp, the \emph{learnt arguments, attacks, and supports} by $\AGa$% during the AX
    .
\end{itemize} 
\end{definition}
Note that,
by definition of AXs, all learnt arguments, attacks and supports are from the (%relevant applicable 
corresponding) exchange BAFs.  
Note also that in Example~\ref{ex:exchange} neither agent learns anything\cut{ in the AX therein
}, as indeed each contributed an attack/support
        already present in the other agent's private QBAF.  %As a result, that exchange has no effect on either agent's stance on the explanandum. 
        \cut{Next we illustrate
        agents' learning
        during AXs% can affect         the agents'stances
        .}

\begin{example}\label{ex:learning}
Let us extend the AX from Example~\ref{ex:exchange} to obtain 
$\langle %\BAFu, 
\BAFi^0, \BAFi^1, \BAFi^{2}, \Agents^0, \Agents^1, \Agents^2, \SpeakerM \rangle$ %(partially visualised in the top two rows of Figure \ref{fig:structure}) 
such that (see the top two rows of Figure \ref{fig:structure}):
\begin{itemize}
        \item $\BAFi^2 = \langle \{ \arge, \arga, \argb, \argc \}, \{(\arga,\arge) \}, \{ (\argb,\arge), (\argc, \arga) \} \rangle$ 
        \item $\AGm^2 = \AGm^1 =\AGm^0
        $; $\AGh^2$ is such that $\QBAFh^2 \sqsupset \QBAFh^1$ where $\ArgsH^2 = \ArgsH^1 \cup \{ \argc \}$, $\AttsH^2 = \AttsH^1$, $\SuppsH^2 = \SuppsH^1 \cup \{ (\argc,\arga) \}$ and $\BSh^2(\argc) = 0.2$;
        \item $\SpeakerM((\argc, \arga)) = (\AGm,2)$, namely $\AGm$ contributes the support $(\argc,\arga)$ in $\BAFi^2$ at timestep 2. 
        \end{itemize}
%\FT{Note that, if $\AGb$ had instead  assigned a lower bias to  the learnt argument $\argc$ at timestep 2, this argument (and the learnt support) would have had a more marginal effect on the evaluation of $\arge$ in $\AGb$'s private QBAF, possibly not enough to change the stance. For illustration, with $\BSb^2(\argc) = 0.2$, then, the argument evaluations are $\SFb^2(\QBAFb^2,\arge) = 0.7$,         $\SFb^2(\QBAFb^2,\arga) = 0.35$ and         $\SFb^2(\QBAFb^2,\argc) = 0.2$,  and thus $\StanceB^2(\arge) = +$ still, as in Examples \ref{ex:agents} and \ref{ex:exchange}. }
\end{example}

We will impose that any attack/support  which is added to the \public\ BAF by an agent is \emph{learnt} by the other agents, alongside any new arguments %in
introduced by those attacks/supports.
Thus,
for any %agent 
$\AGa \!\!\in \!\!\Agents$ %, for any timestep 
and $t\!>\!0$, $\BAFi^t \!\setminus\! \BAFi^{t-1} \!\sqsubseteq \!\QBAFa^{t}$.
However, \cut{note that} agents have a choice on their biases on the learnt arguments.
These biases could reflect, e.g., their trust on the %agent contributing the arguments 
contributing agents or the intrinsic quality of the arguments. 
Depending on these biases,  learnt attacks and supports  may influence the agents' stances on the explanandum differently. For %example
illustration, in %the earlier example
Example~\ref{ex:learning}, $\AGh$ opted for a low bias (0.2% in [0.1]
) on the learnt argument $\argc$, %which could be, for example, due to the human not believing the machine's argument regarding the actor's poor performance to be valid,
resulting in $\SFh^2(\QBAFh^2,\arge) \!=\! 0.648$,         $\SFh^2(\QBAFh^2,\arga) \!=\! 0.32$ and         $\SFh^2(\QBAFh^2,\argc) \!=\! 0.2$,  and thus $\StanceH^2(\arge) \!= \! +$ still, as in Examples~\ref{ex:agents}, ~\ref{ex:exchange}. If, instead, $\AGh$ had chosen a high bias on the new argument, e.g. $\BSh^2(\argc) = 1$, this would have given $\SFh^2(\QBAFh^2,\arge) = 0.432$,         $\SFh^2(\QBAFh^2,\arga) = 0.88$ and         $\SFh^2(\QBAFh^2,\argc) = 1$, leading to $\StanceH^2(\arge) = -$, thus resolving the conflict\cut{ in the agents' stances}.
This illustration shows that learnt attacks, supports and arguments may fill gaps% in the agents' knowledge
, change agents' stances on explananda %\delete{(in line with \emph{enablement of self-transformation} \cite{McBurney_02})} 
and pave the way to the resolution of conflicts% amongst agents
.

\begin{definition}\label{def:resolution}
    Let $E=\langle %\BAFu,
    \BAFi^0, \ldots, \BAFi^{n},  \Agents^0,\ldots,\Agents^n, \SpeakerM \rangle$ be an AX for explanandum $\arge$ amongst agents %\delete{$\AGa(1), \ldots \AGa(k)$} 
    $\Agents$
    such that  $\StanceA^0(\arge) \neq \StanceB^0(\arge)$ for some $\AGa, \AGb \in \Agents$% \delete{$\{\AGa(1), \ldots \AGa(k)\}$}
    . Then: 
    
    \begin{itemize}
        \item $E$ is 
        \emph{resolved at timestep $t$}, for some $0< t \leq n$, \Iff\ $\forall \AGa, \AGb \in \Agents$, $\StanceA^{t}(\arge) = \StanceB^{t}(\arge)$, and is \emph{unresolved at $t$} otherwise;
        \item $E$ is \emph{resolved} \Iff\ it is resolved at timestep $n$ and it is unresolved at every timestep $0 \leq t <n$; %\todo{it makes sense to me to say that you stop as soon as you resolve, is it? this is a choice, and does not need to be in the definition, it could be an assumption we make oiutside } 
        \item $E$ is \emph{unresolved} \Iff\ it is unresolved at every %timestep 
        $0 %\leq 
        < t \leq n$. %\todo{do we need this notion?}
               \end{itemize}
    \end{definition}

Thus, a resolved AX starts with  a conflict between at least two agents and ends %as soon as 
when no conflicts amongst any of the agents exist or when the agents give up on trying to find a resolution. 
Practically,
AXs may be governed by a \emph{turn-making function} $\turn\!:\! \mathbb{Z}^+ \!\rightarrow \! 2^\Agents$ %imposing that at most one agent contributes
determining which agents should contribute %arguments NO
at any timestep% (we have used an interleaving turn-making function implicitly in all examples so far)
. 
Then, an AX may be deemed to be unresolved if, for example, all agents decide, when their turn comes, against contributing% since the\FT{ir} last contribution% was made
. 
%\delete{no arguments are contributed for two consecutive timesteps, i.e. at $n-1$ and $n$.} \todo{AR: always two? not sure about this} \todo{FT: you are right, it needs to be the number of agents...}
%\todo{AR: I think I would just remove this sentence; we could state it as a possibility but it's up to the designers of the arg exchange really?}

%Note that other notions of resolution may be useful in some settings, e.g. for inquiry agents may want to continue an exchange even after they have resolved all conflicts, as unearthing further conflicts may be useful in the search of truth. 
%
Note that, while agents' biases and evaluations are kept private during AXs, we %implicitly 
assume that agents share their stances on the explanandum% (at every timestep or, equivalently, when these stances change)
, so that they are aware of whether the underpinning conflicts are resolved. 
Agents' stances, when ascertaining whether an AX is resolved, are %solely 
evaluated internally by the agents, %given their private tuples, 
without any %explicit 
shared evaluation of the exchange BAF, unlike, e.g. in \cite{Tarle_22} and other works we reviewed in §\ref{sec:related}.
%\delete{\A{Note also that we define AXs with no distinction between types of agents at this point, but allow for restrictions for more asymmetric cases such as XAI where a machine is explaining to a human, as we will see in §\ref{sec:xai}.}}

Finally, note that our definition of AX is neutral as to the role of agents therein, allowing in particular that agents have symmetrical roles (which is natural, e.g., for inquiry) as well as asymmetrical roles (which is natural, e.g., when machines explain to  humans: this will be our focus from §\ref{sec:xai}).

\section{Explanatory Properties of AXs %for XAI
} 
\label{sec:xai}

%
\iffalse
\todo{something of the below should be in section 4 instead? if at all} \AR{or maybe here?} \FT{It is already somehow in section 4 - it belongs there, as it is a design choice, nothing to do with XAI? }
We are then able to impose that, for any agent $\AGa \in \Agents$ at timestep $t$, $\BAFi^t \setminus \BAFi^{t-1} \sqsubseteq \QBAFa^{t}$ in \AR{AX}s in this setting, i.e. requiring that contributed attacks/supports \emph{must} be incorporated into agents' private QBAFs via learning. 
The acceptance or rejection of the learnt arguments in an agent is modelled with the assignment of a high or low, \resp, bias from the agent.
Another option would have been to define \AR{AX}s without requiring that relations added to the \public\ BAF are incorporated to agents' private QBAFs, but the results are (initially) identical if the argument has no effect on other arguments' evaluations.
Our reasoning is that our approach allows for a learnt argument which has been initially rejected by an agent to be ``revived'' by more reasoning in the form of supporting arguments added at future timesteps, which would not be possible (or would be, at least, less elegant) if learning was based on an argument's inclusion (or not) in the private QBAFs.
\fi

 Here  we focus on 
 %positioning argumentative exchanges 
singling out desirable properties that  AXs may need satisfy to support interactive XAI.
Let us assume as given an AX %. Given the t\^ete-\`a-t\^ete nature of %most explanations in XAI, we limit the exchanges to those 
$E=\langle \BAFi^0,\ldots, \BAFi^{n}, \Agents^0, \ldots, \Agents^n, \SpeakerM \rangle$ for $\arge$ as in Definition~\ref{def:exchange}.
%
%\todo{add arguments as to why these properties are useful for (interactive)  - I have started here and there}
%
%
%\subsection{Properties \delete{of Argumentative Exchanges}}
%
The first three properties impose basic requirements on AXs so that they result in fitting explanations% for XAI landscape where a machine is explaining to a human
.
%They are all satisfied by design, as we will see.

\begin{property} 
\label{prop:connectedness}
$E$ \emph{satisfies}
    \emph{connectedness} 
    \Iff\ 
    for any $0 \leq t\leq n$, if
    $|\ArgsI^t| > 1$ then $\forall \arga \in \ArgsI^t%\FT{\setminus \{\arge\}} ????
    $, $\exists \argb \in \ArgsI^t$ such that $(\arga, \argb) \in \AttsI^t \cup \SuppsI^t$ or $(\argb, \arga) \in \AttsI^t \cup \SuppsI^t$.
\end{property}
Basically, connectedness imposes that there should be no floating arguments and no ``detours'' in the  exchange BAFs, at any stage during the AX.  It is  linked to directional connectedness in \cite{Cyras_22}\cut{, where an argument's strength depends only on the strengths of arguments with a path to it, since we %impose arguments are relevant to the explanandum
focus on (Q)BAFs \emph{for %the explanandum
$\arge$}}. A violation of this property would lead to counter-intuitive (interactive) explanations, with agents seemingly ``off-topic''.

\begin{property} 
\label{prop:acyclicity}
$E$ \emph{satisfies} \emph{acyclicity}
    \Iff\ for any $0 \leq t\leq n$, $\nexists \arga \in \ArgsI^t$ such that $\argpaths(\arga,\arga) \neq \emptyset$.
\end{property}

Acyclicity %\delete{prevents circular reasoning} 
ensures that all reasoning is directed towards the explanandum in AXs. A violation of this property %\delete{would again}
may lead to seemingly non-sensical (interactive) explanations.

\begin{property} 
\label{prop:contributor}
$E$ \emph{satisfies}
    \emph{contributor irrelevance} 
    \Iff\ for any  AX for $\arge$ $\langle {\BAFi^0}', \ldots, {\BAFi^{n}}', {\Agents^0}', \ldots, {\Agents^n}', \SpeakerM' \rangle$, if ${\BAFi^0}' \!= \! \BAFi^0$,  ${\BAFi^n}' \!= \! \BAFi^n$% and 
    , ${\Agents^0}' \!=\! \Agents^0$, then $\forall \AGa \in \Agents$: $\StanceA^n(\QBAFa^n,\arge) \!=\! \StanceA^n({\QBAFa^n}',\arge)$. 
    %\todo{do we need to state that learnt arguments are given the same bias or is this implicit since they're the same agents? Good point - but tricky to address - let us go with your interpretation and ignore the matter} \AR{agreed}
\end{property}

Contributor irrelevance ensures that the same final exchange BAF \cut{, and thus, the same reasoning in the interactive explanation,} results in the same stances for all agents, regardless of the contributors of its %arguments 
attacks and supports or the order in which they were contributed. 

These three properties are basically about the exchange BAFs in AXs, and take the viewpoint of an external ``judge'' for the explanatory nature of AXs. 
These basic properties are all satisfied, by design, by AXs:\footnote{Proofs for all propositions are in %\A{\href{https://arxiv.org/abs/2303.15022}{arxiv.org/abs/2303.15022}%
the supplementary material.}

\begin{proposition}\label{propos:basic}
    Every AX satisfies %connectedness, acyclicity and contributor irrelevance
    Properties 1 to 3. %and resolution existence.
\end{proposition}

We now introduce %some intuitive but more complex 
 properties which AXs may not always satisfy, but which, nonetheless, may be desirable if AXs are to generate meaningful (interactive) explanations. First, we define  notions of \emph{pro} and \emph{con} arguments in AXs, %equating 
amounting to positive and negative reasoning towards the explanandum.

\begin{definition}\label{def:procon}
Let $\BAF = \langle \Args, \Atts, \Supps \rangle$ be any BAF for $\arge$. Then, the
\emph{pro arguments} and \emph{con arguments} for  $\BAF$ are, \resp:

%\begin{itemize}
    %
    %\item 
    $\bullet \Pros(\BAF) \!=\! %\{ \arge \} \cup 
    \{ \arga \!\in\! \Args | \exists p \!\in\! \argpaths(\arga,\arge), \text{ where } | p \cap \Atts | \text{ is even}
    %\exists t' \in \{ 1, \ldots, t-1 \}, \exists \argb \in \ArgsI^t $ such that $\SpeakerM((\arga,\argb)) = (\AGa,t')$ and $ \AGa$ $\text{ is arguing for } \arge \text{ at } t' 
    \}$;
    
    %\item 
    $\bullet \Cons(\BAF) \!=\! 
    \{ \arga \!\in \!\Args | \exists p \!\in \! \argpaths(\arga,\arge), \text{ where } | p \cap \Atts | \text{ is odd}
    %\exists t' \in \{ 1, \ldots, t-1 \}, \exists \argb \in \ArgsI^t $ such that $  \SpeakerM((\arga,\argb)) = (\AGa,t')$ $\text{ and } \AGa$ $\text{ is arguing against } \arge \text{ at } t' 
    \}$.
    %
%\end{itemize}
\end{definition}

Note that the intersection of pro and con arguments may %not 
be non-empty as %we allow 
 multiple paths to explananda may exist, so an argument may bring %constitute 
both positive and negative %reasoning
reasoning.

Pro/con arguments with an even/odd, \resp, number of attacks in their path to $\arge$ are related to chains of supports (\emph{supported}/\emph{indirect defeats}, \resp) in \cite{Cayrol:05} (we leave the study of formal %relationships
links to future work). %It is easy to see that 
Pro/con arguments are responsible for increases/decreases, \resp, in $\arge$'s strength using DF-QuAD%, as follows.
:

\begin{proposition}\label{propos:paths}
    For any $\AGa \in \Agents$, let $\SFa$ indicate the evaluation method by DF-QuAD. 
    Then, for any $0 <t \leq n$:
    
    %\begin{itemize}
     %   \item 
     $\bullet$ if $\SFa(\QBAFa^t,\arge) > \SFa(\QBAFa^{t-1},\arge)$, then $\Pros(\BAFi^t) \supset \Pros(\BAFi^{t-1})$;
        
        %\item 
        $\bullet$ if $\SFa(\QBAFa^t,\arge) \!<\! \SFa(\QBAFa^{t-1},\arge)$, then $\Cons(\BAFi^t) \supset \Cons(\BAFi^{t-1})$.
    %\end{itemize}
    % Let, $\forall \AGa \in \Agents$, $\SFa$ be the DF-QuAD semantics. 
    % Then, for any $\AGa \in \Agents$ and $\ArgsA^t = \ArgsA^{t-1} \cup \{ \arga \}$ where $t>1$: \todo{invert this}
    % \begin{itemize}
    %     \item if $\arga \in \Pros(\BAFi^t) \setminus \Cons(\BAFi^t)$, then $\SFa^{t}(\arge) \geq \SFa^{t-1}(\arge)$;
    %     %
    %     \item if $\arga \in \Cons(\BAFi^t) \setminus \Pros(\BAFi^t)$, then $\SFa^{t}(\arge) \leq \SFa^{t-1}(\arge)$.
    % \end{itemize}
\end{proposition}

We conjecture (but leave to future work) that this result (and more later)  holds for other gradual semantics  satisfying \emph{monotonicity} \cite{Baroni_19} or \emph{bi-variate monotony/reinforcement} \cite{Amgoud_18}.

%\AR{  It should be noted that the intersection of the pro and con arguments may not necessarily be empty, since we allow multiple paths from an argument to the explanandum, and so an argument may constitute both positive and negative reasoning. However, the properties which follow are defined based on \emph{some} positive or negative reasoning existing.}

\begin{property}\label{prop:resolution}
    $E$ satisfies \emph{resolution representation} iff $E$ is resolved and $\forall \AGa \in \Agents$: 
    %%%
    if %$\exists \AGa \in \Agents$, where 
        $\StanceA^n(\arge) > \StanceA^0(\arge)$, then 
        $\Pros(\BAFi^n) \neq \emptyset$; and
        %$\exists \arga \in \ArgsI^n$ such that $\exists p \in \argpaths(\arga,\arge)$, where $| p \cap \AttsI^n |$ is even; and 
        %
        %%%%
        if %$\exists \AGa \in \Agents$, where 
        $\StanceA^n(\arge) < \StanceA^0(\arge)$, then 
        $\Cons(\BAFi^n) \neq \emptyset$.
        %$\exists \arga \in \ArgsI^n$ such that $\exists p \in \argpaths(\arga,\arge)$, where $| p \cap \AttsI^n |$ is odd.
    
\end{property}

This property also takes the viewpoint of an external ``judge'', by imposing that 
the final exchange BAF convincingly represents a resolution of the conflicts between agents' stances, thus %demonstrating
showing why stances were changed. Specifically, it imposes that a changed stance must be the result of pro or con arguments (depending on how \cut{agents'} stances have changed). For example, in %the private QBAFs $\QBAFm^3$ and $\QBAFh^3$ in 
Figure \ref{fig:structure}, $\argb$, $\argd$, $\argf$ are pro arguments which could justify an \cut{agent's} increase in stance for $\arge$, while  $\arga$, $\argc$ are con arguments which could justify its decrease.
Note that this property does not hold in general, e.g., given an agent which (admittedly counter-intuitively) increases its evaluation of arguments when they are attacked. 
%
%\AR{In order to undertake theoretical analysis of these properties, we will assume that all agents evaluate arguments with the DF-QuAD semantics. While equating machine and human reasoning thus is somewhat unrealistic, this semantics holds a number of intuitive properties for this setting (see \emph{balance} and \emph{monotonicity}, as well as their implications, in \cite{Baroni_19}) and we believe that this approximation is an acceptable compromise for initial analysis. We leave to future work comprehensive studies of formal properties of different semantics in \AR{AX}s, e.g. their effectiveness in explanation and how closely they model human reasoning.}
However, it holds for %certain choices of 
some evaluation models, notably DF-QuAD again:

\begin{proposition}\label{propos:resolution}
    If $E$ is resolved and 
    $\forall \AGa \in \Agents$, %:
    %\begin{itemize}
        %\item 
        \cut{the evaluation method} 
        $\SFa$ is  DF-QuAD,
        %\item $\forall \arga \in \ArgsA^n \setminus \{ \arge \}$, $|\argpaths(\arga,\arge)| = 1$;
    %\end{itemize}
    then $E$ satisfies resolution representation.
\end{proposition}

%\cut{Again we conjecture that this proposition also holds %be the case not only for DF-QuAD but also 
%for other gradual semantics satisfying \emph{monotonicity} \cite{Baroni_19} or \emph{bi-variate monotony/reinforcement} \cite{Amgoud_18}.} %However, we leave a comprehensive study of semantics and properties in AXs to future work.

The final property we consider concerns unresolved AXs, in the same spirit as resolution representation. 

\begin{property}\label{prop:conflict}
    $E$ satisfies \emph{conflict representation} \Iff\ $E$ is unresolved, $\Pros(\BAFi^n) \neq \emptyset$ and $\Cons(\BAFi^n) \neq \emptyset$.
    %and:
    %\begin{itemize}
    %    \item $\exists \arga \in \ArgsI^n$ such that $\exists p \in \argpaths(\arga,\arge)$, where $| p \cap \AttsI^n |$ is even; and 
        %
    %    \item $\exists \arga \in \ArgsI^n$ such that $\exists p \in \argpaths(\arga,\arge)$, where $| p \cap \AttsI^n |$ is odd. 
    %\end{itemize}
\end{property}

This property thus requires that the conflict %\cut{(and its causes)} 
in an unresolved AX %\delete{\A{(and from where it arose)}} 
is apparent in the exchange BAF, namely it includes %both chains of reasoning discussed for Property \ref{prop:resolution} 
both pro and con arguments (representing the conflicting stances).
For example, if the AX in Figure \ref{fig:structure} concluded unresolved at $t\!=\!2$, this property requires that $\BAFi^2$ contains both pro arguments \cut{arguing} for $\arge$ (e.g. $\arga$ or $\argc$) and con arguments \cut{arguing} against it (e.g. $\argb$).
%\cut{Again, there may be links here to \emph{internal} and \emph{external coherence} \cite{Cayrol:05}, which we leave as future work.}
%
This property does not hold in general, e.g. for an agent who rejects all arguments by imposing on them minimum biases and contributes no attack or support%: arguments constituting the reasoning for this agent's stance will not exist in the exchange BAF
. %In fact, proving that this property is satisfied  
Proving that this property holds requires consideration of the agents' behaviour, which we  examine next. %in §\ref{sec:behaviour}. 

\section{Agent %\delete{s'} 
Behaviour in AXs for XAI %\todo{since §\ref{ssec:behave} is Agent Behaviour}
}
\label{sec:agents}

All our examples so far have
illustrated how AXs may support explanatory interactions amongst a \emph{machine} $\AGm$ and a \emph{human} $\AGh$. This specific XAI setting is our focus in the remainder,
%i.e. from now on 
where
we  assume $\Agents=\{\AGm,\AGh\}$.
Also, 
for simplicity, we impose (as in all illustrations) that $\RangeM = \RangeH = [0,1]$, $\RNegM = \RNegH = [0,0.5)$, $\RNeuM = \RNeuH = \{ 0.5 \}$ and $\RPosM = \RPosH = (0.5, 1]$.
We also
restrict attention to AXs  governed by a turn-making function
$\turn$ 
imposing a strict interleaving %where at most one agent can contribute arguments at each timestep  and 
such that 
$\turn(i)=\{\AGm\}$ if $i$ is odd, and $\turn(i)=\{\AGh\}$ otherwise (thus, in particular, the machine starts the interactive explanation process% at the start of the exchange and the agents take turns
).

%In the remainder, unless stated otherwise, we  assume an \AR{AX} $\FT{E}=\langle \BAFi^0,\ldots, \BAFi^{n}, \Agents^0, \ldots, \Agents^n, \SpeakerM \rangle$ for $\arge$  \FT{as in Definition~\ref{def:exchange} but amongst $\Agents = \{ \AGm, \AGh \}$ and governed by $\turn^*$}. 

 In line with standard argumentative XAI, the machine may draw the QBAF in its \emph{initial private triple} (at $t=0$) from the model it is explaining\cut{ to the human}. This QBAF may be obtained by virtue of some abstraction methodology %, e.g. one amongst those overviewed in \cite{Cyras_21}, 
 or may be the basis of the model itself (%for intrinsically argumentative models, again 
 see \cite{Cyras_21}).
The humans, instead, may  draw the QBAF in their initial private triple, for example, from their own knowledge, biases, and/or regulations on the expected machine's behaviour. 
The decision on the evaluation method, for machines and humans, may be dictated by specific settings and desirable agent properties therein.
Here we focus on how to formalise and evaluate interactive explanations between a machine and a human using AXs, and ignore how their initial private triples are obtained.  % \todo{add something like this in the intro? may be actually in related work, adding a short paragraph on arg expl...} 

%In the remainder of this section 
Below we
 define various %\cut{possible} 
 behaviours dictating how
 machines and humans can engage in AXs  for XAI, focusing on ways to i) determine their biases and ii)  decide their contributions (attacks/supports) to (unresolved) AXs.%\FT{, and in §\ref{sec:evaluation} we will experiment with various combinations of these behaviours and choices of evaluation methods. } 
%
%\subsection{Agent Goals}
\label{ssec:goals}
\label{sec:behaviour}
\label{ssec:behave}

%We now consider  ways in which a machine $\AGm$ and a human $\AGh$ behave when the machine is explaining to the human. This behaviour is captured by  how the agents decide 1) the biases of learnt arguments, and 2)  which attacks and supports (and relative arguments) to contribute to unresolved \AR{AX}s. The agents' behaviour in turn defines how the human-machine interaction takes place from the viewpoint of \AR{AX}s.  \AR{We will also present theoretical analysis of agent behaviour in the form of propositions demonstrating \AR{AX}s' suitability for XAI.}

\iffalse
\todo{
To be defined in this section:
\begin{itemize}
    \item different constant biases in machines DONE
    \item random bias in humans DONE
    \item confirmation bias in humans DONE
    \item simple machines DONE
    \item greedy machines DONE
    \item counterfactual machines DONE
    \item static humans DONE
    \item argumentative humans? NO?
\end{itemize}
}
\fi

\textbf{Biases.}
As  seen in §\ref{sec:exchanges}, the degree to which learnt attacks/supports impact the stances of agents on explananda % in the context of their private QBAFs 
is determined by the agents' biases on the %corresponding 
learnt arguments. 
In  XAI %there are several 
different considerations regarding this learning apply %which are different for 
to machines and humans%, as we will now outline
.
Firstly, not all machines may be capable of learning: simple AI systems which provide explanations but do not have the functionality for understanding any input from humans are common in %SOTA-X
AI%, as we discussed
. 
% THIS SEEMS DANGEROUS, WITHOUT A RESULT SAYING THAT NO EFFECT....DOES IT NOT DEPEND ON THE SEMANTICS? DUMMY? We model such systems as agents which assign minimum biases to learnt arguments such that they have no effect on the agents' evaluations of other arguments in their private QBAFs.\footnote{\todo{Something here about the behaviour of the semantics we experiment with later}} 
%
% \begin{definition}\label{def:unintelligent} For any \emph{\todo{unintelligent/oblivious/ignorant?close-minded?} machine agent} $\AGm$ at timestep $t$ and learnt argument $\arga \in \ArgsM^t \setminus \ArgsM^{t-1}$, $\BSm^t(\arga) = 0$. \end{definition}
%
Secondly, machines %which are 
capable of learning may assign different biases to the learnt arguments: a low bias indicates scepticism %towards the learnt argument 
while a high bias indicates credulity. Machines may be designed to give low biases %when arguments are presented by
to arguments from sources which cannot be trusted, e.g. when the expertise of a human is deemed insufficient% \cut{due to poor previous feedback}
,
or high biases to arguments when the %feedback is to be taken as being correct
human is deemed competent, e.g. 
\iffalse when %a human is 
debugging %an explanation method
a system
\fi
in debugging.
%In this paper, 
Here, we refrain from accommodating such %\cut{variety of learning} 
challenges and
focus instead on 
% a general case for XAI where  machines assign constant biases to arguments learnt from humans (which% may, admittedly, be too rigid in some settings but which 
 %, albeit restrictive, serves  as a sensible starting point)
 the restrictive (but sensible, as a starting point) case where machines assign constant biases to arguments %learnt 
 from humans.

\begin{definition}\label{def:machine bias}
Let $c \in [0,1]$ be a chosen constant. For any  learnt argument $\arga \in \ArgsM^t \setminus \ArgsM^{t-1}$ at timestep $t$, $\BSm^t(\arga) = c$.
\end{definition}
If $c\!=\!0$ then the machine is unable to learn, whereas $0\!<\!c\!<\!1$ gives partially sceptical machines and $c\!=\!1$ gives credulous machines. %which buy the human feedback in full (gradations are also possible)}.
The choice of $c$ thus depends on the specific setting of interest, and may have an impact on the conflict resolution desideratum for AXs. 
For example, let %us assume that $\AGm$ uses 
$\AGm$ use DF-QuAD as its evaluation method: 
%In the first case, setting the bias of all learnt arguments to the minimum value means that 
%\cut{Then,  if $c=0$ then $\AGm$'s evaluation of, and thus stance on, the explanandum will not change. %This demonstrates the difficulty in achieving consensus between agents when one or all agents are incapable of learning. 
%Meanwhile, in the second case where no arguments in an agent's private QBAF are assigned the minimum or maximum biases, 
%Instead, if $0<c<1$, any learnt argument will have an effect on %all other downstream arguments' (including 
%the explanandum's evaluation%, guaranteeing an effect of learnt arguments on an explanandum
%, opening the possibility of reaching consensus. 
%Finally,} %in the third case where all learnt arguments are assigned the maximum bias,  
if $c\!=\!1$  we can derive guarantees of rejection/weakening or acceptance/strengthening of arguments which are attacked or supported, \resp, by learnt arguments,\footnote{Propositions on such effects 
%guarantees 
are in %\A{\href{https://arxiv.org/abs/2303.15022}{arxiv.org/abs/2303.15022}}%the \FTreb{arxiv version}
the supplementary material.} demonstrating the potential (and dangers) of credulity in machines (see §\ref{sec:evaluation}).

Humans, meanwhile, %will 
typically assign varying biases to arguments based on their own internal beliefs. 
%In order to model this fairly, we propose that the distribution of biases for learnt arguments be Gaussian with the mean of $0.5$ to give a fair assignment representing a human's internal beliefs.
These assignments %of biases 
%may not always be fair, however. Humans are subject to a host of logical fallacies and  cognitive biases, hindering attempts in earnest at rational judgement. For example,
may reflect cognitive biases such as the \emph{confirmation bias} \cite{Nickerson_98} -- the tendency towards looking favourably %towards 
at evidence which supports one's prior views.
In %our simulated experiments in 
§\ref{sec:evaluation} we model humans so that they assign random biases to learnt arguments, but %accommodate 
explore confirmation bias by %\delete{aligning arguments with the human stance on the explanandum (to a varying degree, using an offset and reducing their average bias).}
applying a constant offset to reduce the bias assigned by the human% to each learnt argument
.
%Given that, here, the machine's and the human's stances are in conflict until and unless the exchange is resolved, we model humans with confirmation bias by including an offset to the mean of the Gaussian distribution for any learnt arguments, i.e. reducing the average evaluation of these arguments compared to those which represent the agent's prior beliefs.\footnote{}
This differs, e.g., from the modelling of confirmation bias in \cite{Tarle_22},  %where this acts 
acting on the probability of an argument being learned.
We leave the exploration of alternatives for assigning biases %\cut{to learnt arguments} 
to future work.

\textbf{Attack/Support Contributions.}
%Concerning agents' choice\AR{s} of attacks/supports to contribute in unresolved \AR{AX}s, 
%\todo{``Simple behaviours'' Thanks, we will use “shallow”.}
We consider %\cut{three alternatives,  referred to as} 
\emph{shallow}, \emph{greedy} and \emph{counterfactual} behaviours: intuitively, the first %\cut{roughly} 
corresponds to the one-shot explanations in most XAI%, where a single set of attackers or supporters of \AR{(i.e. positive or negative evidence towards)} the explanandum is contributed. Greedy behaviour will \emph{dynamically}, i.e. based on the most recent state of the \AR{AX}, 
, the second contributes the (current) strongest argument in favour of the agent position% without considering its effect on the exchange BAF. Meanwhile, counterfactual behaviour considers, again dynamically, 
, and the third considers how each attack/support may (currently) affect the exchange BAF before it is %\cut{actually} 
contributed.
%We characterise these behaviours in terms of two states in which agents may find themselves%, defined as follows.
All behaviours identify %(sets of) 
argument pairs to be added to the exchange BAF as attacks or supports reflecting their role in the private QBAFs from which they are drawn.
We use the following notion:

\begin{definition}\label{def:states}
%In an argumentative exchange $\langle \BAFu, \BAFi^0, \ldots, \BAFi^{n}, \Agents^0, \ldots, \Agents^n, \SpeakerM \rangle$ for $\arge$ amongst agents $\Agents = \{ \AGm, \AGh \}$ 
For %\AR{$\langle %\BAFu, \BAFi^0, \ldots, \BAFi^{n}, \Agents^0, \ldots, \Agents^n, \SpeakerM \rangle$} 
$E$ resolved at timestep $t$, if $\StanceM^t(\arge) \!>\! \StanceH^t(\arge)$ then the \emph{states of}  $\AGm$ and $\AGh$ \emph{at} $t$ are, \resp, %described as 
\emph{arguing for %$\arge$
}
and 
\emph{arguing against $\arge$} (else, the states are reversed). 
\end{definition}
The agents' states point to a %velleity of 
%attempt at
``window for
persuasion'', whereby an agent arguing for (against) $\arge$ may wish to attempt to increase (decrease, \resp) the stance of the other agent, without accessing their private QBAFs, thus differing from other works, e.g.  %simplifying the problem since we choose not to include a shared evaluation of the exchange BAF (as in, e.g., the merged graph in 
\cite{Tarle_22},
which rely on shared evaluations:
%as we believe this is not realistic in the XAI setting\AR{: 
in our case, reasoning is shared but it is not evaluated in a shared manner.

The \emph{shallow}
 behaviour selects a (bounded by $max$) maximum number of supports for/attacks against the explanandum if the agent is arguing for/against, \resp, it, as follows:

\begin{definition}\label{def:shallow}
Let  $max \in \mathbb{N}$% be a constant
.
Agent $\AGa \in \Agents$ exhibits  \emph{shallow behaviour} (wrt $max$) \Iff, at any %timestep 
$0\leq t<n$ where $\turn(t) = \{\AGa\}$,  $C=\{(\arga, \argb) | \SpeakerM((\arga, \argb)) = (\AGa, t)\}$ is a maximal (wrt cardinality) set $\{ (\arga_1,\arge), \ldots, (\arga_p,\arge) \}$  
 with $p \leq max$  such that: 
\begin{itemize}
    \item if $\AGa$ is arguing for $\arge$ then %$S$ is a maximal (wrt cardinality) set  of $p \leq max$ supports 
    $C \!%= \{ (\arga_1,\arge), \ldots, (\arga_p,\arge) \} 
    \!\subseteq \!\SuppsA^{t-1} \!\setminus \!\SuppsI^{t-1}$ where $\forall i \!\!\in\!\! \{ 1, \ldots,$ $ p \}$,$\nexists (\argb,\arge) \!\!\in\! \SuppsA^{t-1} \!\!\setminus\!\! (\SuppsI^{t-1} \!\cup C)$ %such that 
    with %$\SFa(\QBAFa^{t-1},\argb) > \SFa(\QBAFa^{t-1},\arga_i)$
    $\SFa^{t-1}(\argb) \!>\! \SFa^{t-1}(\arga_i)$;
 \item if $\AGa$ is arguing against $\arge$ then %$S$ is a maximal (wrt cardinality) set of $q \leq max$ attacks 
 $C %= \{ (\arga_1,\arge), \ldots, (\arga_p,\arge) \} 
 \!\!\subseteq \! \AttsA^{t-1} \!\!\setminus \!\AttsI^{t-1}$ where $\forall i \!\in\! \{ 1, $ $\ldots, p \}$,$\nexists (\argb,\arge) \!\in\! \AttsA^{t-1} \!\setminus \!(\AttsI^{t-1} \!\cup \!C)$ with %$\SFa(\QBAFa^{t-1},\argb) > \SFa(\QBAFa^{t-1},\arga_i)$
 $\SFa^{t-1}(\argb) \!>\! \SFa^{t-1}(\arga_i)$. 
\end{itemize}
\end{definition}

This behaviour thus focuses on reasoning for or against the explanandum $\arge$ exclusively% \AR{and} so does not consider any \AR{deeper reasoning, nor any} feedback from the other agent
. 
It selects supports or attacks in line with the agent's stance on %the explanandum 
$\arge$ and with the highest evaluation in the contributing agent's private QBAF. %, i.e. the agent considers it the strongest evidence (up to $max$ supporters or attackers) supporting their stance.
This behaviour is inspired by 
  \emph{static} explanation methods in %SOTA-
  XAI, which deliver all \cut{explanatory} information in a single contribution%, often without the capacity for incorporating user feedback. Such explanations make use of evidence for or against (or both, e.g. as in SHAP \cite{Lundberg_17}) the explanandum, i.e. its supporters and/or attackers and nothing more%, as is the case with our simple behaviour
  . Clearly, if %, in an AX, 
  we let $\AGm$ exhibit this shallow behaviour  %, in which the machine is contributing exclusively attackers or supporters of the explanandum, 
  and $\AGh$ be \emph{unresponsive}, i.e.  never contribute any attack/support, then the AX cannot satisfy %\cut{the property of} 
  conflict representation.

The \emph{greedy} behaviour
%\AR{An agent may choose to assume that any pro argument is in support of the explanandum, while any con argument is against it, accepting the possibility of an argument being both pro and con (and the consequential counter-intuitive behaviour which could result).}
allows  an agent arguing for %\cut{the explanandum} 
$\arge$ to 
\cut{either} 
support the pro %arguments 
or attack the con arguments, while that arguing against %the explanandum 
can 
%\cut{either} 
support the con %arguments 
or attack the pro arguments.

\begin{definition}\label{def:greedy}
Agent $\AGa \in \Agents$ exhibits \emph{greedy behaviour} \Iff, at any %timestep 
$0\leq t<n$  where $\turn(t) = \{ \AGa \}$, $C=\{(\arga,\argb) | \SpeakerM((\arga,\argb)) = (\AGa,t)\}$ is empty or amounts to a single 
attack or support $(\arga,\argb) \in (\AttsA^{t-1} \cup \SuppsA^{t-1}) \setminus (\AttsI^{t-1} \cup \SuppsI^{t-1})$ %where $\argb \in \ArgsI^{t-1}$ and:
such that% the following three conditions hold
:
\begin{enumerate}
    \item if $\AGa$ is arguing for $\arge$ then:
%
 %       $\bullet$ 
 $(\arga,\argb) \in \SuppsA^{t-1} $ and $\argb \in \Pros(\BAFi^{t-1}) \cup \{ \arge \}$; or
%
 %       $\bullet$
        $(\arga,\argb) \in \AttsA^{t-1}$ and $\argb \in \Cons(\BAFi^{t-1})$;
    if $\AGa$ is arguing against $\arge$ then:
    %
    %$\bullet$ 
    $(\arga,\argb) \in \SuppsA^{t-1}$ and $\argb \in \Cons(\BAFi^{t-1})$; or
    %    
     %   $\bullet$ 
     $(\arga,\argb) \in \AttsA^{t-1}$ and $\argb \in \Pros(\BAFi^{t-1})\cup \{ \arge \}$;
    \item $\nexists (\arga',\argb') \in (\AttsA^{t-1} \cup \SuppsA^{t-1}) \setminus (\AttsI^{t-1} \cup \SuppsI^{t-1})$ satisfying 1. such that %$\SFa(\QBAFa^{t-1},\arga') > \SFa(\QBAFa^{t-1},\arga)$
    $\SFa^{t-1}(\arga') > \SFa^{t-1}(\arga)$;
    \item $\nexists (\arga'',\argb'') \in (\AttsA^{t-1} \cup \SuppsA^{t-1}) \setminus (\AttsI^{t-1} \cup \SuppsI^{t-1})$ satisfying 1. such that %$\SFa(\QBAFa^{t-1},\arga'') = \SFa(\QBAFa^{t-1},\arga)$ 
    $\SFa^{t-1}(\arga'') = \SFa^{t-1}(\arga)$ 
    and \\ $|argmin_{P'' \in \argpaths(\arga'',\arge)}|P''| \; | < |argmin_{P \in \argpaths(\arga,\arge)}|P| \; |$.
\end{enumerate}
\end{definition}

%\todo{something about not being able to guarantee the results but relying on pro/con def anyway?}

%Here, we  force agents to contribute  a maximum of one argument per turn for simplicity%, since they can make multiple contributions at subsequent turns
%. 

Intuitively,  1. requires that the attack or support, if any, is in line with the agent's views; 2. ensures that the attacking or supporting argument has maximum strength;  and 3. ensures that it is ``close'' to the explanandum% (we allow for any selection among the arguments satisfying these conditions)
.
We posit that enforcing agents to contribute %a maximum of 
at most one argument per turn will aid \emph{minimality} without affecting conflict resolution negatively wrt the shallow behaviour (%assessed empirically in 
see §\ref{sec:evaluation}).
%\todo{may be MOVE here the comment as to why minimality is nice to have?}
Minimality is a common property of explanations in XAI, deemed  beneficial both from a machine perspective, e.g. wrt computational aspects (see \emph{computational complexity} in \cite{Sokol_20}), and  from a human perspective, e.g. wrt cognitive load and privacy maintenance (see \emph{parsimony} in \cite{Sokol_20})%, so long as it does not interfere with the primary goal of conflict resolution
. 
Naturally, however, conflict resolution in AXs should always take precedence over minimality, as prioritising the latter would force AXs to remain empty.

\begin{proposition}\label{propos:greedy}
    If $E$ is unresolved and  $\forall \AGa \!\in\! \Agents$:
    %$\bullet$  
    $\AGa$ exhibits greedy behaviour and
        %\item $\forall \AGa \in \Agents$, $\SFa$ is the DF-QuAD algorithm;
        %\item $\forall \AGa \in \Agents$, $\forall \arga \in \ArgsA^n \setminus \{ \arge \}$, $|\argpaths(\arga,\arge)| = 1$; 
        %$\bullet$ %$\forall \AGa \in \Agents$, 
        $\{\! (\!a,\!b\!) \!\!\in\!\! \AttsI^n\!\cup\! \SuppsI^n | \SpeakerM((a,b))\!\!=\!\!(\AGa,t), t\!\in\! \{1, \ldots, n\} \}\! \neq\! \emptyset$,
   %\noindent 
   then $E$ satisfies conflict representation.
\end{proposition}

%\section{Novel Agent Behaviour for XAI}
%\label{sec:method}

\begin{proposition}\label{propos:equivalence}
    %simple, greedy and counterfactual behaviours align when depth is 1
    %$\forall \AGa \in \Agents$, let $\SFa$ be  DF-QuAD. Then, if $\forall\AGa \in \Agents$,  for all $0\leq t< n$ such that $\forall \arga \in \ArgsA^{t}$ $\argpaths((\arga,\arge)) = \{ (\arga,\arge) \}$  and \FT{$\forall (\arga,\argb)  \in \AttsA^{t}\cup \SuppsA^t$} $\SpeakerM((a,b))=(\AGa,t)$, the simple (with $max = 1$), greedy and counterfactual behaviours align.
  If $\forall \AGa \in \Agents$, %$\SFa$ is DF-QuAD and 
  for all $0\leq t< n$ and $\forall \arga \in \ArgsA^{t}$, $\argpaths((\arga,\arge)) = \{ (\arga,\arge) \}$, then 
  the shallow (with $max = 1$) and greedy behaviours are %equivalent
  aligned.
\end{proposition}

The greedy behaviour may not always %allow to resolve conflicts, as illustrated next.
lead to resolutions:

\begin{example}\label{ex:greedy}
Let us extend the AX from Example~\ref{ex:learning} to $\langle %\BAFu, 
\BAFi^0, \ldots, \BAFi^{3}, \Agents^0, \ldots, \Agents^3, \SpeakerM \rangle$ such that (%partially visualised in 
see Figure~\ref{fig:structure}):
\begin{itemize}
        \item $\BAFi^3 \!=\! \langle \{ \arge, \arga, \argb, \argc, \argd %, \argf 
        \}, \{(\arga,\arge),(\argd,\arga) \}, \{ (\argb,\arge), (\argc, \arga) %, (\argf, \argb) 
        \} \rangle$; 
        \item $\AGh^3 = \AGh^2$; $\AGm^3$ is such that $\QBAFm^3 \sqsupset \QBAFm^2$ where $\ArgsM^3 = \ArgsM^2 \cup \{ \argd %, \argf 
        \}$, $\AttsM^3 = \AttsM^2  \cup \{ (\argd,\arga) \}$, $\SuppsM^3 = \SuppsM^2 %\cup \{ (\argf,\argb) \}
        $, $\BSm^3(\argd) = 0.6$%, $\BSm^3(\argf) \!=\!  0.5$
        ; then, the argument evaluations are $\SFm^3(\QBAFm^3,\arge) = 0.42$,
        $\SFm^3(\QBAFm^3,\arga) = 0.8$,
        $\SFm^3(\QBAFm^3,\argb) = 0.4$,
        $\SFm^3(\QBAFm^3,\argc) = 0.6$,
        $\SFm^3(\QBAFm^3,\argd) = 0.6$;
        \item $\SpeakerM((\argd, \arga)) = 
        (\AGh,3)$, i.e. $\AGh$ contributes attack $(\argd,\arga)$  
        at %timestep 
        $t=3$. 
        \end{itemize}
    Here, in line with the greedy behaviour, $\AGm$ learns the attack $(\argd,\arga)$ contributed by $\AGh$ at timestep 
    3. Then, even if $\AGm$ assigns the same bias to these learnt arguments as $\AGh$ (which is by no means guaranteed), this is insufficient to change the stance, i.e. $\StanceM^3(\arge) = -$, and so the AX remains unresolved. %Thus, since $\StanceA^3(\arge) =\StanceB^3(\arge)$, we can see that the unresolved argumentative exchange from Example \ref{ex:learning} has been resolved here at timestep 3.
\end{example}

The final \emph{counterfactual} behaviour 
takes greater consideration %than the greedy behaviour 
of the argumentative structure of the reasoning available to the agents in order to maximise the chance of conflict resolution with %the minimum n
a limited number of arguments contributed.  
%
%Our final behaviour, with greater consideration of the information which had been contributed to the exchange, can allow \AR{$\AGh$ in this case} to reach a resolution while minimising the number of contributed arguments in doing so. 
This behaviour is defined in terms of the following notion% of how an agent perceives an exchange
.

\begin{definition}\label{def:view}
Given an agent $\AGa \in \Agents$% in an explanatory exchange $\langle \BAFu, \BAFi^0, \ldots, \BAFi^{n}, \Agents^0, \ldots, \Agents^n, \SpeakerM \rangle$ for $\arge$
, \emph{a private view of the \public\ BAF by $\AGa$ at timestep $t$} is any
$\ViewA^t = \langle  \ArgsVA^t, \AttsVA^t, \SuppsVA^t, \BSva^t \rangle$ such that $ \BAFi^t \sqsubseteq \ViewA^t \sqsubseteq \QBAFa^t$.
       % \item $\forall \arga \in \ArgsVA^t$, $\BSva^t(\arga) = \BSa^t(\arga)$.
\end{definition}

% exchange \subset \ViewA^t \subset private
 
%It is easy to see that any $\ViewA^t$ is a QBAF for $\arge$ and that, for $t >0$, $\ViewA^{t-1} \sqsubseteq \ViewA^t$.
An agent's private view of the \public\ BAF thus projects their private biases onto the BAF% resulting from the argumentative exchange with the other agents
, while also potentially accommodating \emph{counterfactual reasoning} with additional arguments. %, naturally capturing,  we believe,  how a human would see an interactive explanation \emph{in motion}. 
%As for which biases are projected onto the private view of the \public\ BAF, this may be a direct mapping of those in the agent's private QBAF, or the agent may decide attempt to view the debate \emph{from the eyes of the other agent}, using an approximation of their opinions. We do this by weighting arguments depending on their contributor, as will be defined shortly.
%
Based on arguments' evaluations in an agent's private view% of the \public\ BAF
, the agent can then judge which attack or support \emph{it perceives} will be
 the most effective.

\begin{definition}\label{def:novel}
    Given an agent $\AGa \in \Agents$% in an explanatory exchange $\langle \BAFu, \BAFi^0, \ldots, \BAFi^{n}, \Agents^0, \ldots, \Agents^n, \SpeakerM \rangle$ for $\arge$
    , 
    $\AGa$'s \emph{perceived effect on $\arge$}  at %timestep 
    $0\!<\!t\!\leq \!n$ of any $(\arga, \argb) \in (\AttsA^{t-1}  \cup \SuppsA^{t-1}) \setminus (\AttsI^{t-1} \cup \SuppsI^{t-1})$, where $\arga \in \ArgsA^{t-1} \setminus \ArgsI^{t-1}$ and $\argb \in \ArgsI^{t-1}$, is $\eff((\arga, \argb),\QBAFa^t
    %\FT{t}
    )=\SFa(\ViewA^t, \arge) - \SFa(\ViewA^{t-1}, \arge)$ %\todo{with $t$ there was no differentiation between the perceived effect of $(\arga,\argb)$ for $\AGa$ as opposed to $\AGb$} 
    for $\ViewA^t \sqsupset \ViewA^{t-1}$  a private view of the \public\ BAF at $t$ by $\AGa$ such that $\ArgsVA^t \!=\! \ArgsVA^{t-1} \cup \{ \arga \}$, $\AttsVA^t = (\ArgsVA^t \times \ArgsVA^t)  \cap  \AttsA^{t-1}$ and $\SuppsVA^t = (\ArgsVA^t \times \ArgsVA^t)  \cap  \SuppsA^{t-1}$. 
    %\begin{itemize}
     %   \item \delete{if $(\!\arga,\! \argb\!) \!\in\! \AttsA^{t-1} \!\setminus\! \AttsI^{t-1}$ then $\AttsVA^t \!=\! \AttsVA^{t-1} \!\cup\! \{ \!(\!\arga,\! \argb) \!\}$ and $\SuppsVA^t \!=\! \SuppsVA^{t-1}$;}
        %
      %  \item \delete{if $(\arga, \argb) \in \SuppsA^{t-1} \setminus \SuppsI^{t-1}$ then $\AttsVA^t = \AttsVA^{t-1}$ and $\SuppsVA^t = \SuppsVA^{t-1} \cup \{ (\arga, \argb) \}$.}
    %\end{itemize}
    
\end{definition}

%By taking a private view of the \public\ BAF, agents thus have a way of determining the perceived effect of adding an attack or support, % giving another parameter which can be used when selecting attacks or supports, and can project whichever biases they choose
%projecting  their biases on the arguments therein. 
The counterfactual view underlying this notion of perceived effect relates to  \cite{Kampik_22}, although we consider the effect of adding an attack or support, whereas they consider an argument's \emph{contribution} by removing it% from a framework
. It also relates to the \emph{hypothetical value} of \cite{Tarle_22}, which however amounts to the explanandum's evaluation in the shared graph.

\begin{definition}\label{def:counterfactual}
Agent $\AGa \!\!\in\!\! \Agents$ exhibits \emph{counterfactual behaviour} \Iff,
at any %timestep 
$0\!\!\leq \!\!t\!\!<\!\!n$  where $\turn(t) \!\! = \!\!\{\AGa\},C\!=\!\{(\arga,\argb) | \SpeakerM((\arga,\argb)) \!=\! (\AGa,t)\}$ %\todo{do we need $C=$ here? NO, BUT IT BREAKS THE FLOW AND MATCHES OTHER DEFS...}
is empty or %amounts to %a single 
%attack/support 
is $\{(\arga,\argb)\}% \in (\AttsA^{t-1} \cup \SuppsA^{t-1}) \setminus (\AttsI^{t-1} \cup \SuppsI^{t-1})
$ 
such that:
% at timestep $t$ where $\turn(t) = \AGa$ contributes an attack or support $(\arga,\argb)$ where $\SpeakerM((\arga,\argb)) = (\AGa,t)$ and:

        $\bullet$ if $\AGa$ is arguing for $\arge$ then $\eff((\arga, \argb),\QBAFa^t) > 0$ and
        $(\arga,\argb)$ is 
        
        \hspace*{0.2cm} $ argmax_{(\arga',\argb')  \in (\AttsA^{t-1} \cup \SuppsA^{t-1}) \setminus (\AttsI^{t-1} \cup \SuppsI^{t-1})}\eff((\arga', \argb'),\QBAFa^t)$;
       
        $\bullet$ if $\AGa$ is arguing against $\arge$ then $\eff((\arga, \argb),\QBAFa^t) \!<\! 0 $ and
        $(\arga,\!\argb)$ 
        
        \hspace*{0.2cm} is $ argmin_{(\arga',\argb') \in (\AttsA^{t-1} \cup \SuppsA^{t-1}) \setminus (\AttsI^{t-1} \cup \SuppsI^{t-1})}\eff((\arga', \argb'),\QBAFa^t)$.

% where:
%     \begin{itemize}
%         \item \emph{simple argumentative behaviour} requires that $\forall \arga \in \hat{\ArgsA^t}$, $\hat{\BSa^t}(\arga) = \BSa^t(\arga)$;
%         \item \emph{strategic argumentative behaviour} 
%         requires that $\forall \arga \in \hat{\ArgsA^t}$, $\hat{\BSa^t}(\arga) = m_\AGa+p_\AGa$ if $\SpeakerM((\arga,\argb)) = (\AGb,t')$ for some $\argb \in \ArgsI^t$, $\AGb \neq \AGa$ and $t' \neq t$, and  $\hat{\BSa^t}(\arga) = m_\AGa-n_\AGa$ otherwise, where $m_\AGa$ is $\AGa$'s \emph{modelled mean bias}, $p_\AGa$ is $\AGa$'s \emph{modelled positive offset} and $\AGa$'s \emph{modelled negative offset}. \todo{fix based on what HL has done}
%     \end{itemize}
%
\end{definition}

\cut{Attacks and supports are thus selected based on the effect   they will have on the AX, as perceived by the contributing agent. }
Identifying attacks and supports based on their effect on the explanandum is related 
to \emph{proponent} and \emph{opponent arguments} in \cite{Cyras_22},  defined  however in terms of \emph{quantitative dispute trees} for BAFs%here arguments are separated based on their contributions 
\cut{towards the root argument (the explanandum, in our case)}.

The counterfactual behaviour 
%may lead to better consider 
may better consider argumentative structure, towards resolved AXs, as shown % in the following example, in which the AX is now resolved.
next.

\begin{example}\label{ex:novel}
Consider the %argumentative
AX from Example~\ref{ex:greedy} but where:
\begin{itemize}
        \item $\BAFi^3 \!=\! \langle \{ \arge, \arga, \argb, \argc, %\argd , 
        \argf 
        \}, \{(\arga,\arge) %,(\argd,\arga) 
        \}, \{ (\argb,\arge), (\argc, \arga), (\argf, \argb) 
        \} \rangle$; 
        \item $\AGm^3$ is such that $\QBAFm^3 \sqsupset \QBAFm^2$ where $\ArgsM^3 = \ArgsM^2 \cup \{ %\argd, 
        \argf 
        \}$, $\AttsM^3 = \AttsM^2  %\cup \{ (\argd,\arga) \}
        $, $\SuppsM^3 = \SuppsM^2 \cup \{ (\argf,\argb) 
        \}$, %\AR{$\BSa^3(\argd) = 0.6$}; 
        $\BSm^3(\argf) =  0.5$; 
        then, the argument evaluations are $\SFm^3(\QBAFm^3,\arge) = 0.546$,
        $\SFm^3(\QBAFm^3,\arga) = 0.92$,
        $\SFm^3(\QBAFm^3,\argb) = 0.7$,
        $\SFm^3(\QBAFm^3,\argc) = 0.6$ and 
        $\SFm^3(\QBAFm^3,\argf) = 0.5$; %and
        %$\SFa^3(\QBAFa^3,\argf) = 0.5$;
        \item $\SpeakerM((\argf, \argb)) = (\AGh,3)$, i.e.  $\AGh$ contributes %the attack $(\argd,\arga)$ only this time and 
        support $(\argf,\argb)$ 
        to $\BAFi^3$% at timestep 3
        .
        \end{itemize}
Here, $\AGh$ contributes $(\argf, \argb)$ in line with the counterfactual behaviour as $\eff((\argf, \argb),\QBAFh^3) = 0.24 > \eff((\argd, \arga),\QBAFh^3) = 0.216$.
%$\AGa$ learns the attack $(\argf,\argb)$ and assigns the same bias to the learnt argument as $\AGb$.
This sufficiently modifies $\AGm$'s private QBAF such that $\StanceM^3 = +$, and
the AX is now resolved: 
the counterfactual behaviour %results in a resolved  \AR{AX} 
succeeds 
where the greedy behaviour did not (Example~\ref{ex:greedy}).
\end{example}

%THIS SENTENCE IS THE PROBELM - MAY BE SOME FUNNY FONT WE DO NOT SEE, TRY REWRITING IT point (1) demonstrates that $(\arga,\argb) \in x$ 

%\todo{if this one doesn't perform well, possibly include another version where learnt arguments are given a high bias in the view only (i.e. taking an assumption that the other agent is rational and is submitting its strongest args)}

%\todo{We also consider...}

%\begin{definition}
%For any \emph{biased machine agent} $\AGm^t$ at timestep $t$ and learnt argument $\arga \in \QBAFm^t \setminus \QBAFm^{t-1}$...
%\end{definition}

We end showing some conditions under which conflict representation is satisfied by the counterfactual behaviour% and under which the described behaviours align
.

\begin{proposition}\label{propos:cf}
    If $E$ is unresolved and is such that $\forall \AGa \in \Agents$:
    %$\bullet$ 
    $\AGa$ exhibits counterfactual behaviour;
        %$\bullet$ 
        $\SFa$ is DF-QuAD;
        %\item $\forall \AGa \in \Agents$, $\forall \arga \in \ArgsA^n \setminus \{ \arge \}$, $|\argpaths(\arga,\arge)| = 1$; 
        %$\bullet$ 
        $\{ (a,b) \in \AttsI^n\cup \SuppsI^n | \SpeakerM((a,b))=(\AGa,t), t\!\in \!\{1, \ldots, n\} \} \!\neq\! \emptyset$;
    %\noindent 
    then $E$ satisfies conflict representation.
\end{proposition}

%\todo{
%\begin{proposition}\label{prop:nonempty}
%    simple, greedy and counterfactual behaviours never have empty contributions when something exists which can be contributed
%\end{proposition}
%\begin{proof}
%    this
%\end{proof}
%}

%\cut{I think this is nice but it'll probably have to go, SM or journal? JOURNAL ... Finally, we use agents' private views to demonstrate how the exchange BAF represents an aggregated explanation as to why each agent arrives at their final stance (whether they are in agreement or not since they project their own biases onto it). A standard explanation, e.g. SHAP, gives a representation of why the machine arrived at a final stance, whereas this demonstrates the benefits of AXs as this applies to all agents. 
%}

%\cut{
%\begin{proposition}\label{propos:viewpoint}
%$E$ is such that $\forall \AGa \in \Agents$ with its private view of the \public\ BAF at timestep $t$ $\ViewA^t = \langle  \ArgsVA^t, \AttsVA^t, \SuppsVA^t, \BSva^t \rangle$: 
%\begin{itemize}
%     \item if $\SFa(\QBAFa^t,\arge) < \SFa(\QBAFa^0,\arge)$ then $\SFa(\ViewA^{t},\arge) < \SFa(\QBAFa^0,\arge)$;
%     %
%     \item if $\SFa(\QBAFa^t,\arge) > \SFa(\QBAFa^0,\arge)$ then $\SFa(\ViewA^{t},\arge) > \SFa(\QBAFa^0,\arge)$.
% \end{itemize}
% \end{proposition}
% \begin{proof}
%     to do
% \end{proof}
% }

%%%%%%%%%%%%%%%%%%%%%%%%%%%%%%%%%%%%%%%%%%%%%%%%%%
%%%%%%%%%%%%%%%%%%%%%%%%%%%%%%%%%%%%%%%%%%%%%%%%%%

\section{Evaluation}
\label{sec:evaluation}

We now evaluate sets of AXs obtained from the behaviours %for XAI 
from §\ref{sec:behaviour} via simulations, %measuring how well AXs achieve the goals discussed in §\ref{ssec:goals} 
%when the machine and human adopt the different behaviours 
using the following metrics:\footnote{%Exact formulations of these measures are given in 
See the supplementary material 
%\FTreb{arxiv version} 
%\A{\href{https://arxiv.org/abs/2303.15022}{arxiv.org/abs/2303.15022}} 
for exact formulations.}

\textbf{Resolution Rate (\RR)}: 
%For a given set of AXs, 
the proportion of resolved AXs%, directly assessing the goal of conflict resolution
.

\textbf{Contribution Rate (\CR):} 
%For a given set of AXs, 
the average number of arguments contributed to the exchange BAFs in the resolved AXs, in effect measuring the total information exchanged\cut{ and assessing to what extent minimality is achieved}.

\textbf{Persuasion Rate (\PR):}
for %a given set of AXs and 
an agent% therein
, the proportion of resolved AXs in which the agent's initial stance  is %equal to 
the other agent's final stance, %demonstrating how often the given agent persuades the other agent
measuring the agent's persuasiveness.

\textbf{Contribution Accuracy (\CA):}
%For a given set of AXs and an agent therein,
for an agent, the proportion of the %agent 
contributions which%were among those which
, if the agent was arguing for (against) $\arge$, would have maximally increased (decreased, \resp) $\arge$'s strength in the other agent's private QBAF. %, indirectly assessing the goal of conflict resolution. 

 We tested \PR\ and \CA\ for machines only.
Let \emph{unresponsive behaviour} amount to contributing nothing %\cut{to the AX}
(as in §\ref{sec:behaviour})%\cut{ and  an X agent be an agent with X behaviour}
. Then, our hypotheses were:

%\begin{itemize}
    %
    \textbf{H1:} For a \textit{shallow machine} %explaining to 
    and
    an \textit{%\delete{static}
    unresponsive human}, as the \textit{\emph{max} constant} increases, %effectiveness increases
    \RR, \CR\ and \CA\ increase. %\emph{increase $n$ in Definition \ref{def:static} using no confirmation bias, human gives no input here and it's just one contribution of $n$ arguments}

    \textbf{H2:} For a \textit{shallow machine} %explaining to 
    and
    an \textit{%\delete{static}
    unresponsive human}, as the human's \textit{confirmation bias} increases, %effectiveness 
    \RR\ decreases.
    %As a , a machine's greedy {\bf static} behaviour becomes less effective and less efficient. %\emph{for a range of $n$ options, increase confirmation bias, again no input from the human}
    %
    
    \textbf{H3:} For a \textit{greedy machine} %explaining to 
    and
    a \textit{counterfactual human}, %effectiveness 
    \RR\ increases \textit{relative to a shallow machine} and an
    \textit{unresponsive human}.
    %\textbf{Greedy, dynamic behaviour} is more effective and efficient vs. static behaviour. %\emph{test the two greedy strategies (with humans who are also greedy) vs. a range of or the best performing $n$ options}
    %\item balanced explanations give less consensus vs. those which are biased (though they may inspire trust, which we can't model)
    %
    
    \textbf{H4:} For a \textit{greedy machine} %explaining to 
    and
    a \textit{counterfactual human}, as the machine's \textit{bias on learnt arguments} increases, {\RR} increases while {\CR\ and \PR} decrease. %\emph{with the greedy strategies (and humans who are also greedy), compare unintelligent vs. sceptical vs. gullible. we need to check there are at least some instances where its being solved due to different strategy, rather than just the machine being able to change its stance}

    \textbf{H5:} For a \textit{counterfactual machine} %explaining to 
    and
    a \textit{counterfactual human}, \RR\ and \CA\ increase %effectiveness and efficiency increases 
    \textit{relative to a greedy machine}. %\emph{with the greedy strategies (and huma doing, we just need to find a way of making it work!} 
    %
    %\item[H6] \textbf{Strategic argumentative behaviour} gives more effective and efficient explanations vs. simple argumentative behaviour.
    %

%These hypotheses capture a thorough assessment of both argumentative exchanges and the behaviours we have introduced. 
 
\textbf{Experimental Setup.}
%We will now detail the empirical evaluation of our six stated hypotheses. 
For each AX for $\arge$ (restricted as in §\ref{sec:agents}), we created a ``universal BAF'', i.e. a BAF for $\arge$ of which all % other
argumentation frameworks are subgraphs. 
We populated the universal BAFs %were populated 
with 30 arguments by first generating a 6-ary tree with $\arge$ as the root. 
Then, any argument other than $\arge$ had a 50\% chance of having a directed edge towards a random previous argument in the tree, to ensure that multiple paths to the explanandum are present. 50\% of the edges in the universal BAF were randomly selected to be attacks, and the rest to be supports. %\todo{try without multiple paths?}
We built %argumentative
agents' private QBAFs from the universal BAF by performing a random traversal through the universal BAF and stopped when the QBAFs reached 15 arguments, selecting a random argument from each set of children, as in \cite{Tarle_22}. 
We then assigned random biases to arguments in the agents' QBAFs, and %the agents were assigned 
(possibly different) random evaluation methods to agents amongst  QuAD \cite{Baroni_15}, DF-QuAD, REB \cite{Amgoud_17_BAF} and QEM \cite{Potyka_18} (all  %\cut{operating} 
with evaluation range $[0,1]$). We %considered 
used different evaluation methods %\cut{in order to remove any semantics-specific phenomena and} 
to simulate different %evaluation methods 
ways to evaluate arguments  in real-world humans/machines. 
%\todo{``An original aspect of the approach is that agents may use different evaluation (although in the scope of gradual semantics). However the take-away message is not clear here. In the experiments agents are randomly assigned an evaluation method among 4 possible one, but what are the conclusions?’’ Our intention is to demonstrate that AXs are semantics-agnostic, and so applicable to humans and machines who evaluate arguments differently. We will clarify this.}
We repeated this %universal BAF generation 
process  till  agents held different stances on $\arge$.

For each hypothesis, we ran 1000 experiments per configuration, %using random seeds in $\{1,\ldots, 1000\}$ to make 
making sure the experiments for different strategies are run %in the same settings
with the same QBAFs. 
%Harry: The purpose of the random seed is to make sure the experiments for different strategies are run in the same settings. So e.g. for greedy vs argumentative, the two strategies are ran in games with random seed from 1 to 1000.  The same random seed produces the same QBAFs. In this way we know that any differences in outcome can only be caused by the difference in strategies, not by random noise in QBAF generation
%\todo{, sequentially setting the random seed to integers in the range $[1,1000]$ to ensure consisten\AR{t (Q)BAFs} across \AR{different} settings}.
We ran the simulations on the NetLogo platform\cut{\footnote{\href{http://ccl.northwestern.edu/netlogo/}{http://ccl.northwestern.edu/netlogo/}}} %(Citation: Wilensky, U. 1999. NetLogo. http://ccl.northwestern.edu/netlogo/. Center for Connected Learning and Computer-Based Modeling, Northwestern University. Evanston, IL.) 
using %its BehaviorSpace tool
BehaviorSpace.\footnote{%Code at
See %\href{https://github.com/CLArg-group/argumentative_exchanges}{https://github.com/CLArg-group/argumentative\_exchanges}
\href{https://github.com/CLArg-group/argumentative_exchanges}{github.com/CLArg-group/argumentative\_exchanges}.}
%\AR{For H1, H2 and H4, we performed logistic regression for the binary measures (\RR\ and \PR) and linear regression for the continuous measures (\CR\ and \CA). For H3 and H5, we performed Fisher's exact test for the binary measure (\RR) and Student's t-test for the continuous measure (\CA).} 
We tested the significance between testing conditions in a pairwise manner using the chi-squared test for the discrete measures \RR\ and \PR, and Student's t-test for the continuous measures \CR\ and \CA.
We rejected the null hypotheses when $p<0.01$.

\textbf{Experimental Results.}
Table \ref{table:results} reports the results of our simulations: all hypotheses were (at least partially) verified. 
%H1 to \AR{H4} were verified while \AR{H1 and H5 was only \todo{partially verified}}.

\begin{table}[t]
\begin{center}
\begin{tabular}{ccccccccc}
\cline{2-9}
%\multirow{2}{*}{\!\!\!\textbf{Hyp.}\!\!\!} 
&
%\multirow{2}{*}{\textbf{Setup}} &
\multicolumn{2}{c}{\textbf{Behaviour}} &
\multicolumn{2}{c}{\textbf{Learning}} &
\multirow{2}{*}{\!\!\textbf{\RR}\!\!} &
\multirow{2}{*}{\!\!\textbf{\CR}\!\!} &
\multirow{2}{*}{\!\!\textbf{\PR$_\AGm$}\!\!} &
\multirow{2}{*}{\!\!\textbf{\CA$_\AGm$}\!\!} \\
 &
% &
$\AGm$ &
$\AGh$ &
$\AGm$ &
$\AGh$ &
 &
 &
 &
 \\
\hline
%
% \delete{H1 (Size)} &
% single path, constant prior bias &
% S (1) &
% - &
% 0 &
% 0 &
% 13\% &
% ?\% \\
% \delete{H1 (Size)} &
% single path, constant prior bias &
% S (10) &
% - &
% 0 &
% 0 &
% 51\% &
% ?\% \\
% \delete{H1 (Size)} &
% multiple paths, constant prior bias &
% S (1) &
% - &
% 0 &
% 0 &
% 13\% &
% 1.00 \\
% \delete{H1 (Size)} &
% multiple paths, constant prior bias &
% S (5) &
% - &
% 0 &
% 0 &
% 51\% &
% 4.42 \\
 %(Size, S) 
&
%multiple paths, GauGSian &
S (1) &
- &
- &
0 &
5.4 &
1 &
100 &
45.4 \\
 %(Size, SS) 
&
%multiple paths, Gaussian &
S (2) &
- &
- &
0 &
9.6 &
1.96 &
100 &
51.9 \\
\!\!H1\!\! %(Size) 
&
%multiple paths, Gaussian &
S (3) &
- &
- &
0 &
13.0 &
2.76 &
100 &
56.7 \\
 %(Size) 
&
%multiple paths, Gaussian &
\textbf{S (4)} &
\textbf{-} &
\textbf{-} &
\textbf{0} &
\textbf{13.9} &
\textbf{3.22} &
\textbf{100} &
\textbf{58.1} \\
 %(Size) 
&
%multiple paths, Gaussian &
S (5) &
- &
- &
0 &
13.7 &
3.38 &
100 &
58.3 \\
% H1 %(Size) 
% &
%multiple paths, Gaussian &
% SS (6) &
% - &
% 0 &
% 0 &
% 13.3 &
% 3.35 &
% 100 &
% 0 \\
\hline
% \delete{H2 (Confirmation Bias)} &
% single path, constant prior bias &
% SS (8) &
% - &
% 0 &
% 0 &
% 48 &
% ? \\
% \delete{H2 (Confirmation Bias)} &
% single path, constant prior bias &
% SS (8) &
% - &
% 0 &
% -0.2 &
% 32 &
% ? \\
% %
% \delete{H2 (Confirmation Bias)} &
% multiple paths, constant prior bias &
% SS (5) &
% - &
% 0 &
% 0 &
% 48 &
% 4.41 \\
% %
% \delete{H2 (Confirmation Bias)} &
% multiple paths, constant prior bias &
% SS (5) &
% - &
% 0 &
% -0.2 &
% 31 &
% 4.48 \\
%
 %(Confirmation Bias, SS) 
%&
% S (4) &
% - &
% - &
% 0 &
% 13.9 &
% 3.22 &
% 100 &
% 58.1 \\
%
 %(Confirmation Bias) 
&
%multiple paths, Gaussian &
S (4) &
- &
- &
-0.1 &
11.2 &
3.26 &
100 &
57.6 \\
\!\!\multirow{2}{*}{H2}\!\! %(Confirmation Bias) 
&
%multiple paths, Gaussian &
\textbf{S (4)} &
\textbf{-} &
\textbf{-} &
\textbf{-0.2} &
\textbf{8.6} &
\textbf{3.27} &
\textbf{100} &
\textbf{58.0} \\
%
 %(Confirmation Bias) 
&
%multiple paths, Gaussian &
S (4) &
- &
- &
-0.3 &
6.7 &
3.30 &
100 &
58.3 \\
%
 %(Confirmation Bias) 
&
%multiple paths, Gaussian &
S (4) &
- &
- &
-0.4 &
5.3 &
3.38 &
100 &
58.5 \\
\hline
% H3 (EA) &
% %multiple paths, Gaussian &
% EA &
% EA &
% 0 &
% -0.2 &
% 25.8 &
% 6.79 &
% 86.4 &
% 13.6 \\
% H3 (EA) &
% %multiple paths, Gaussian &
% EA ($\leq$3) &
% EA &
% 0 &
% -0.2 &
% 12.4 &
% 3.21 &
% 82.3 &
% 17.7 \\
% H3 (EA) &
% %multiple paths, Gaussian &
% EA ($\leq$4) &
% EA &
% 0 &
% -0.2 &
% 16.3 &
% 4.12 &
% 85.3 &
% 14.7 \\
% H3 (EA) &
% %multiple paths, Gaussian &
% SS (4) &
% EA &
% 0 &
% -0.2 &
% 8.5 &
% 3.20 &
% 100 &
% 0 \\
 %(EA) 
% &
% S (4) &
% - &
% - &
% -0.2 &
% 8.6 &
% 3.27 &
% 100 &
% 58.0 \\
% \!\!\multirow{2}{*}{H3}\!\! %(EA) 
&
%multiple paths, Gaussian &
\!\!\!\!G ($\leq$3)\!\!\!\! &
C &
0 &
-0.2 &
9.8 &
3.15 &
83.7 &
38.8 \\
 %(G) 
\!\! H3\!\! &
%multiple paths, Gaussian &
\!\!\!\!G ($\leq$4)\!\!\!\! &
C &
0 &
-0.2 &
11.9 &
3.88 &
79.0 &
37.1 \\
 %(G) 
&
%multiple paths, Gaussian &
\textbf{G} &
\textbf{C} &
\textbf{0} &
\textbf{-0.2} &
\textbf{18.8} &
\textbf{7.16} &
\textbf{79.3} &
\textbf{35.7} \\
\hline
% H4 (Learning) &
% %multiple paths, Gaussian &
% G &
% G &
% 0 &
% -0.2 &
% 25.8 &
% 6.79 &
% 86.4 &
% 13.6 \\
% H4 (Learning) &
% %multiple paths, Gaussian &
% G &
% G &
% 0.5 &
% -0.2 &
% 47.6 &
% 6.79 &
% 42.2 &
% 57.8 \\
% H4 (Learning) &
% %multiple paths, Gaussian &
% G &
% G &
% 1.0 &
% -0.2 &
% 59.8 &
% 5.51 &
% 27.9 &
% 72.1 \\
 %(Learning) 
% &
% G &
% C &
% 0 &
% -0.2 &
% 18.8 &
% 7.16 &
% 79.3 &
% 35.7 \\
\!\!\multirow{2}{*}{H4}\!\! %(Learning) 
&
%multiple paths, Gaussian &
\textbf{G} &
\textbf{C} &
\textbf{0.5} &
\textbf{-0.2} &
\textbf{42.2} &
\textbf{6.73} &
\textbf{31.5} &
\textbf{37.5} \\
 %(Learning) 
&
G &
C &
1.0 &
-0.2 &
55.5 &
5.24 &
20.4 &
38.2 \\
\hline
% H5 (G) &
% %multiple paths, Gaussian &
% G &
% G &
% 0.5 &
% -0.2 &
% 47.6 &
% 6.79 &
% 42.2 &
% 57.8 \\
% H5 (G) &
% %multiple paths, Gaussian &
% G &
% G &
% 0.5 &
% -0.2 &
% 47.7 &
% 6.94 &
% 41.3 &
% 58.7 \\
% \!\!\multirow{2}{*}{H5}\!\! %(G) 
% &
% G &
% C &
% 0.5 &
% -0.2 &
% 42.2 &
% 6.73 &
% 31.5 &
% 37.5 \\
 %(G) 
\!\!H5\!\! &
%multiple paths, Gaussian &
C &
C &
0.5 &
-0.2 &
48.4 &
7.37 &
41.5 &
50.5 \\
\hline
% H6 (Strategic) &
% %multiple paths, Gaussian &
% G &
% EA &
% 0.5 &
% -0.2 &
% 47.7\% &
% 6.94 &
% 41.3\% &
% 58.7\% \\
% H6 (Strategic) &
% %multiple paths, Gaussian &
% strategic(MCB 0.1) &
% EA &
% 0.5 &
% -0.2 &
% 47.6\% &
% 6.87 &
% 41.4\% &
% 58.6\% \\
% H6 (Strategic) &
% %multiple paths, Gaussian &
% strategic(MCB 0.2) &
% EA &
% 0.5 &
% -0.2 &
% 47.6\% &
% 6.84 &
% 41.2\% &
% 58.8\% \\
% H6 (Strategic) &
% %multiple paths, Gaussian &
% strategic(MCB 0.3) &
% EA &
% 0.5 &
% -0.2 &
% 47.7\% &
% 6.83 &
% 41.7\% &
% 58.3\% \\
% H6 (Strategic) &
% strategic(MCB 0.4) &
% EA &
% 0.5 &
% -0.2 &
% 47.8\% &
% 6.83 &
% 42.1\% &
% 57.9\% \\
% H6 (Strategic) &
% %multiple paths, Gaussian &
% G &
% G &
% 0.5 &
% -0.2 &
% 47.8\% &
% 6.97 &
% 40.6\% &
% 59.4\% \\
% H6 (Strategic) &
% %multiple paths, Gaussian &
% strategic(MCB 0.1) &
% G &
% 0.5 &
% -0.2 &
% 47.6\% &
% 6.90 &
% 40.8\% &
% 59.2\% \\
% H6 (Strategic) &
% %multiple paths, Gaussian &
% strategic(MCB 0.2) &
% G &
% 0.5 &
% -0.2 &
% 40.4\% &
% 6.88 &
% 40.4\% &
% 59.6\% \\
% H6 (Strategic) &
% %multiple paths, Gaussian &
% strategic(MCB 0.3) &
% G &
% 0.5 &
% -0.2 &
% 47.5\% &
% 6.87 &
% 40.4\% &
% 59.6\% \\
% H6 (Strategic) &
% %multiple paths, Gaussian &
% strategic(MCB 0.4) &
% G &
% 0.5 &
% -0.2 &
% 47.6\% &
% 6.88 &
% 40.8\% &
% 59.2\% \\
% %
% \hline
\end{tabular}
\end{center}
\protect\caption{Results in the simulations for the five hypotheses for three behaviours: \underline{S}hallow (\emph{max} constant given in parentheses)\AR{;} \underline{G}reedy (where %if there is a 
any limit on the number of contributed arguments by the agent %, it is given by the value in parentheses
is in brackets%\delete{, if present}
); and \underline{C}ounterfactual. Learning %corresponds to the constant value assigned to learnt arguments 
amounts to $c$ in Definition~\ref{def:machine bias} for $\AGm$ and to the confirmation bias offset for $\AGh$ (where appropriate). 
We report \RR, \PR$_\AGm$ and \CA$_\AGm$ as percentages.
% the resolution rate (\RR, as a percentage), the contribution rate (\CR), the persuasion rate for the machine (\PR$_\AGm$, as a percentage) and the contribution accuracy for the machine (\CA$_\AGm$, as a percentage).
We indicate in bold the %\FT{best} \AR{(trade-off in)} behaviour\cut{, i.e. which was the most desirable or represented an acceptable trade-off}, 
chosen baseline for the next hypothesis.
%which was then carried over as a baseline for the next hypothesis.}
%We also the percentage of resolved exchanges in which $\AGm$ or $\AGh$ convinced the other agent of its initial stance. 
} \label{table:results}
\end{table}

\textbf{H1:} % The results in Table \ref{table:results} show that, 
As expected, increasing  $max$  for  shallow machines %to human 
results in 
%BENCE \delete{significantly higher \RR\ ($\beta = 2.09$, 95\% CI $[1.47, 2.70]$, $p<0.001$), \CR\ ($\beta = 0.554$, 95\% CI $[0.506, 0.602]$, $p<0.001$) and \CA\ ($\beta = 3.20$, 95\% CI $[2.48, 3.91]$, $p<0.001$).}
%up to $max=3$ ($p<0.005$ for $max$ values of $1$ vs $2$ and $2$ vs $3$ for all metrics)}. \AR{Above this limit, this trend weakened in all three of these metrics, which we initially attributed to} the \AR{limited} size of the universal BAFs used. However, \AR{this was still the case} when we tested with larger universal BAFs, e.g. 10-ary trees resulted in \AR{similar results} at around $max=4$\AR{, indicating that this may be due to the semantics' characteristics}.
significantly higher \RR, \CR\ and \CA\ up to $max=3$ ($p<0.005$ for %all pairwise comparisons between values 
$max$ values of $1$ vs $2$ and $2$ vs $3$
for all metrics). Above this limit ($max$ values of $3$ vs $4$ and $4$ vs $5$), this trend was no longer apparent\cut{ in all three of these metrics}, suggesting that there was a limit to the effectiveness of contributing arguments at this distance from $\arge$.
%which we initially attributed to} the \AR{limited} size of the universal BAFs used. However, \AR{this was still the case} when we tested with larger universal BAFs, e.g. 10-ary trees resulted in \AR{similar results} at around $max=4$\AR{, indicating that this may be due to the semantics' characteristics}.
Note that the machine's \PR\ is always 100\% here, since the (unresponsive) human does not contribute\cut{ to the AXs and so cannot persuade the machine}.%, which mainly validates our experimental set up
%\todo{as the size increases: effectiveness and accuracy up, efficiency down. 100\% resolution since human does not argue here. Low effectiveness generally shows the need for deeper reasoning.}
%\todo{could mention that the first hypotheses validate the setup by testing some parameters like the hypotheses in \cite{Tarle_22}}

%Bence:
%\todo{$R^2$, coefficient/gradient, p}

\textbf{H2:} We fixed %the \emph{max} constant for the simple machine to four
$max=4$ (the value with the maximum \RR\ for H1) and found that increasing the confirmation bias in the human significantly decreased the machine's \RR\ 
%BENCE ($\beta = -21.7$, 95\% CI $[-27.3, -16.1]$, $p<0.001$), 
initially
($p<0.01$ for $0$ vs $-0.1$ and $-0.1$ vs $-0.2$), before the effect tailed off as \RR\ became very low ($p=0.09$ for $-0.2$ vs $-0.3$ and $p=0.03$ for $-0.3$ vs $-0.4$), 
demonstrating the need
for behaviours which consider deeper reasoning %\delete{than 4} %(when it is available) 
than the shallow behaviour to achieve higher resolution rates.
\cut{The intuitive results from H1 and H2 %, while straightforward, 
confirmed \AR{that} the \AR{AX}s were functioning appropriately.} 

\textbf{H3:} From here onwards we tested with a counterfactual human\footnote{%We also experimented with greedy humans %but found that the results 
%\AR{and the findings}
%in all tests were %identical to those shown
%as for the counterfactual human
%were \AR{very similar}.
Experiments with greedy humans gave similar findings\cut{ to those which we describe}.} and fixed the level of confirmation bias therein to $-0.2$. We %then
compared shallow against greedy machines, also limiting the number of arguments they contributed %by the greedy machines 
to maxima of three and four to compare fairly with the shallow machine with the fixed \emph{max} constant\cut{ (note that %when \AR{the human contributes to the exchange (i.e. is not unresponsive), 
\CR\ accounts for human contributions also in addition to those from the machine)}.  
%It can be seen that 
\RR\ increased significantly with the greedy behaviour ($p<0.001$), % as the interactions became more complex, 
%\AR{advantages 
over the shallow machine which remained statistically significant when we restricted the greedy machine's contributed arguments to 4 ($p<0.005$), but not to 3 ($p=0.202$).
\cut{Some of these gains are down to the machine now being able to be persuaded by the human's contributions, but not their entirety, as shown by the machine's \PR\ falling when the restricted greedy method is introduced but remaining stable as the \RR\ increases with the easing restrictions.}
%\AR{It should be noted that although the accuracy drop for the greedy method here\todo{...HL, not sure how to explain this?! I had a note saying that accuracy is not fair to the greedy strategy but I can't remember why?}}
%Note that here the machine still assigns biases of zero to the learnt arguments but it is not uncommon that the resolution results in the machine reaching the initial stance of the human. This is because the human can contribute attacks and supports between arguments which are already in $\AGm$'s private QBAF.
%\todo{resolution isn't 100\% since we contribute edges, not arguments}

\textbf{H4:} %We then used a greedy machine explaining to a counterfactual human and varied the constant biases assigned by the machine to the learnt arguments. 
%The advantages of learning in the machine were stark, with 
\RR\ increased significantly with the bias on learnt arguments %BENCE \AR{($\beta = 36.7$, 95\% CI $[32.6, 40.8]$, $p<0.001$)}.
($p<0.001$ for both comparisons of learning configurations: $0$ vs $0.5$ and $0.5$ vs $1$). 
However, %it \AR{was not without its downsides:}
the machine's \CR\ 
%\AR{($\beta = -2.15$, 95\% CI $[-2.83, -1.47]$, $p<0.001$)} 
and \PR\ 
%\AR{($\beta = -51.1$, 95\% CI $[-58.0, -44.3]$, $p<0.001$)}} 
fell significantly 
($p<0.001$ for similar pairwise comparisons, except for $0$ vs $0.5$ for \CR, where $p=0.27$). 
highlighting the %\delete{gullible}
naive nature of machines learning %without scepticism 
credulously (i.e. assigning all learnt arguments the top bias). 
\cut{This demonstrates, importantly, that interactivity can be a powerful mediator
when machines learn from humans. }
%\AR{Interestingly, the \CA\ was also penalised by this lack of scepticism, further demonstrating this need.}

\textbf{H5:} 
The counterfactual behaviour outperformed the greedy behaviour significantly in terms of both \RR\ ($p<0.01$) and \CA\ (%$t(999)= 1.96$, 
$p<0.001$), %demonstrating
showing, even in this limited setting, the %\cut{clear} 
advantages in taking a counterfactual view, % of the AX 
%\cut{before %selecting an argument to be contributed 
%contributing to the AX}, 
given that the strongest argument (as selected by the greedy behaviour) may not always be the most effective in persuading%\cut{ the other agent}
.
%\delete{In this limited setting where both behaviours will eventually select the same arguments, albeit possibly in different orders, the latter be does not outperform the former in \RR, partially falsifying H5, while the \CR\ and \PR\ for both behaviours are similar. However, the latter's \CA\ is significantly higher (%$t(1000) = 2.55$, $p = 0.011$), demonstrating that taking the counterfactual view is worthwhile even in this constrained setting.}} %Here, the results for the two machine behaviours were almost identical, which shows that no advantages were gained wrt these metrics in deploying the counterfactual over the greedy behaviour. This highlights the need for future work investigating the settings which exhibit the effect shown in Example \ref{ex:novel}. 

% https://www.scribbr.com/statistics/statistical-tests/

% https://www.scribbr.com/statistics/t-test/

% https://www.scribbr.com/statistics/students-t-table/

%%%%%%%%%%%%%%%%%%%%%%%%%%%%%%%%%%%%%%%%%%%%%%%%%%%%%%%%%%%%%%%%%%%%%%%%%%%%%%%%%%%%%%%%%%%%%%%%%%%%

\section{Conclusions% and Future Work
}
\label{sec:conclusions}

We  defined the novel concept of AXs, and deployed AXs in the %interactive 
XAI setting where a machine and a human engage in interactive explanations%.  We showed the potential of AXs in providing argumentative explanations to allow for rich, dynamic interactions 
, powered by non-shallow reasoning, contributions from %multiple 
both agents and modelling of agents' learning and explanatory behaviour.
\cut{In the dominant XAI perspective, the human only consumes the explanation produced by the machine (e.g. using %local explanation 
methods such as SHAP~\cite{Lundberg_17}). Instead, when the machine and the human interact within AXs, both agents can produce and consume information, interactively.}
This work opens several avenues for future  work, besides those already mentioned. 
%We plan to User study justifying all the assumptions, links to trust
It would be interesting to experiment with any number of agents, besides the two %agents only 
that are standard in XAI, and to identify  restricted cases %, specifically for XAI, 
where hypotheses H1-H5 are %analytically
guaranteed to hold.
It would also be interesting to accommodate mechanisms for  machines to model humans, e.g. %building upon work on 
as in opponent modelling%as in
~\cite{Hadjinikolis_13}.  
Also fruitful could be an investigation of how closely AXs can represent machine and human behaviour.
Further, while we used AXs in XAI,  they may be %potentially 
%applicable 
usable in various multi-agent settings%where conflicts need resolving
\cut{, e.g.~\cite{Fan_12,
Fan_12_MD,Tarle_22}}.

\section*{Acknowledgements}

This research was partially funded by the  ERC under the
EU’s Horizon 2020 research and innovation programme (%grant agreement 
No. 101020934, ADIX) and by J.P. Morgan and by the Royal
Academy of Engineering, UK%, under the Research Chairs and Senior Research Fellowships scheme
.  %Any views or opinions expressed herein are solely those of the authors.

%% The file kr.bst is a bibliography style file for BibTeX 0.99c
\bibliographystyle{kr}
\bibliography{bib_short}

\newpage

\appendix
\section{Supplementary Material \\ (Interactive Explanations by Conflict Resolution via Argumentative Exchanges) }

\subsection{Proofs of Propositions in the Main Body}

Proof for Proposition \ref{propos:basic}:
\begin{proof}
    Let $E=\langle \BAFi^0,\ldots, \BAFi^{n}, \Agents^0, \ldots, \Agents^n, \SpeakerM \rangle$ be an AX for $\arge$.
    For Properties \ref{prop:connectedness} and \ref{prop:acyclicity}, $E$ trivially satisfies connectedness and acyclicity given that $\BAFi^n$ is a BAF for $\arge$. 
    For Property \ref{prop:contributor}, consider that $\forall \AGa \in \Agents'$, $\StanceA^n({\QBAFa^n}',\arge)$ depends only on $\SFa^n({\QBAFa^n}',\arge)$ by Definition \ref{def:stance}. 
    By Definition \ref{def:exchange}, we can see that ${\QBAFa^n}' \sqsupset {\BAFi^n}'$ such that ${\ArgsA^n}' = {\ArgsA^0}' \cup {\ArgsI^n}'$, ${\AttsA^n}' = {\AttsA^0}' \cup {\AttsI^n}'$ and ${\SuppsA^n}' = {\SuppsA^0}' \cup {\SuppsI^n}'$, regardless of $\SpeakerM'$, so it can be seen that $\SFa^n({\QBAFa^n}',\arge) = \SFa^n(\QBAFa^n,\arge)$ and $\StanceA^n({\QBAFa^n}',\arge) = \StanceA^n(\QBAFa^n,\arge)$, thus $E$ satisfies contributor irrelevance.
    %For Property \ref{prop:existence}, since the BAF $\BAF'$ is a BAF for $\arge$, then, by Definition \ref{def:exchange}, in $\SpeakerM$ agents may contribute arguments in turn until $\BAFi^n = \BAF'$.
\end{proof}

Proof for Proposition \ref{propos:paths}:
\begin{proof}
    Let Case 1 for $\arge$ be that $\SFa(\QBAFa^t,\arge) > \SFa(\QBAFa^{t-1},\arge)$ and Case 2 for $\arge$ be that $\SFa(\QBAFa^t,\arge) < \SFa(\QBAFa^{t-1},\arge)$.
    It can be seen from Definition \ref{def:exchange} that $\BSa^t(\arga) = \BSa^{t-1}(\arga)$ $\forall \arga \in \ArgsA^t$ for $t>0$. 
    Then, by the definition of DF-QuAD (see §\ref{sec:preliminaries}), Case 1 must be due to: 
    1a.) an added supporter of $\arge$, i.e. $\SuppsA^t(\arge) \supset \SuppsA^{t-1}(\arge)$; 
    1b.) a strengthened supporter of $\arge$, i.e. $\exists \arga \in \SuppsA^t(\arge)$ such that $\SFa(\QBAFa^t,\arga) > \SFa(\QBAFa^{t-1},\arga)$; or
    1c.) a weakened attacker of $\arge$, i.e. $\exists \argb \in \AttsA^t(\arge)$ such that $\SFa(\QBAFa^t,\argb) < \SFa(\QBAFa^{t-1},\argb)$.
    Similarly, and also by the definition of DF-QuAD, Case 2 must be due to: 
    2a.) an added attacker of $\arge$, i.e. $\AttsA^t(\arge) \supset \AttsA^{t-1}(\arge)$; 
    2b.) a weakened supporter of $\arge$, i.e. $\exists \arga \in \SuppsA^t(\arge)$ such that $\SFa(\QBAFa^t,\arga) < \SFa(\QBAFa^{t-1},\arga)$; or
    2c.) a strengthened attacker of $\arge$, i.e. $\exists \argb \in \AttsA^t(\arge)$ such that $\SFa(\QBAFa^t,\argb) > \SFa(\QBAFa^{t-1},\argb)$.
    In Cases 1a and 2a, it can be seen that $\Pros(\BAFi^t) \supset \Pros(\BAFi^{t-1})$ and $\Cons(\BAFi^t) \supset \Cons(\BAFi^{t-1})$, \resp\
    In Cases 1b and 2b, we repeat Cases 1 and 2 for the supporter $\arga$.
    In Cases 1c and 2c, we repeat Cases 1 and 2  for the attacker $\argb$, noting that the $\Pros$ and $\Cons$ sets will be inverted due to the extra attack in the path to $\arge$.
    Since $\QBAFa^t$ is a(n acyclic) QBAF for $\arge$, all paths through the multitree eventually reach leaves and Cases 1a and 2a apply, thus the proposition holds.
\end{proof}

Proof for Proposition \ref{propos:resolution}:
\begin{proof}
     It can be seen, by Definition \ref{def:stance}, that if $\StanceA^n(\arge) > \StanceA^0(\arge)$, then $\SFa(\QBAFa^n,\arge) > \SFa(\QBAFa^0,\arge)$. Then, by Definition \ref{def:exchange} and Proposition \ref{propos:paths}, $\Pros(\BAFi^n) \neq \emptyset$. 
    Analogously, it can be seen, by Definition \ref{def:stance}, that if $\StanceA^n(\arge) < \StanceA^0(\arge)$, then $\SFa(\QBAFa^n,\arge) < \SFa(\QBAFa^0,\arge)$. Then, by Definition \ref{def:exchange} and Proposition \ref{propos:paths}, $\Cons(\BAFi^n) \neq \emptyset$.
    Thus, resolution representation is satisfied.
\end{proof}

Proof for Proposition \ref{propos:greedy}:
\begin{proof}
    Since $E$ is unresolved and by Definition \ref{def:states}, let $\AGa \in \Agents$ be arguing for $\arge$ and $\AGb \in \Agents$ be arguing against $\arge$. 
    By Definition \ref{def:greedy}, we know that 
    $\AGa$ contributed some $(\arga,\argb)$ such that either: 
    $(\arga,\argb) \in \SuppsA^{t-1}$ and $\argb \in \Pros(\BAFi^{t-1}) \cup \{ \arge \}$; or
    $(\arga,\argb) \in \AttsA^{t-1}$ and $\argb \in \Cons(\BAFi^{t-1})$.
    By Definition \ref{def:procon}, it can be seen that in both cases, $\arga \in \Pros(\BAFi^t)$ and so, by Definition \ref{def:exchange}, $\Pros(\BAFi^n) \neq \emptyset$. 
    Similarly, by Definition \ref{def:greedy}, we know that 
    $\AGb$ contributed some $(\arga,\argb)$ such that either: 
    $(\arga,\argb) \in \SuppsA^{t-1}$ and $\argb \in \Cons(\BAFi^{t-1})$; or
    $(\arga,\argb) \in \AttsA^{t-1}$ and $\argb \in \Pros(\BAFi^{t-1}) \cup \{ \arge \}$.
    By Definition \ref{def:procon}, it can be seen that in both cases, $\arga \in \Cons(\BAFi^t)$ and so, by Definition \ref{def:exchange}, $\Cons(\BAFi^n) \neq \emptyset$. 
    Thus, conflict resolution is satisfied.
\end{proof}

Proof for Proposition \ref{propos:equivalence}:
\begin{proof}
    Let $(\arga,\arge)$ be such that $\SpeakerM((\arga,\arge)) = (\AGa, t)$. 
    If $\AGa$ is arguing for $\arge$, then the shallow behaviour, by Definition \ref{def:shallow}, requires that $(\arga,\arge) \in \SuppsA^{t-1}$ and $\nexists (\argb,\arge) \in \SuppsA^{t-1} \setminus (\SuppsI^{t-1} \cup \{ (\arga,\arge) \})$ such that $\SFa(\QBAFa^{t-1},\argb) > \SFa(\QBAFa^{t-1},\arga)$.
    Meanwhile, the greedy behaviour, by Point 1 of Definition \ref{def:greedy}, requires that $(\arga,\arge) \in \SuppsA^{t-1}$, and by Point 2, requires that $\nexists (\argb,\arge) \in \SuppsA^{t-1} \setminus (\SuppsI^{t-1} \cup \{ (\arga,\arge) \})$ such that $\SFa(\QBAFa^{t-1},\argb) > \SFa(\QBAFa^{t-1},\arga)$. (Point 3 is not relevant here since $\argpaths((\argc,\arge)) = \{ (\argc,\arge) \}$ $\forall \argc \in \ArgsA^{t-1}$ at $t>0$.) The greedy behaviour is thus aligned with the shallow behaviour when $\AGa$ is arguing for $\arge$.
    %
    %\delete{Finally, the counterfactual behaviour, by Definition \ref{def:counterfactual}, requires that $(\arga,\arge) = argmax_{(\arga',\arge)  \in (\SuppsA^{t-1}) \setminus ( \SuppsI^{t-1})}\eff((\arga', \arge),\QBAFa^t)$, given that only supporters can increase $\arge$'s evaluation in this restricted case, where $\argpaths((\argc,\arge)) = \{ (\argc,\arge) \}$ $\forall \argc \in \ArgsA^{t}$. Given the definition of DF-QuAD (see §\ref{sec:preliminaries}), it is clear that, here, this is equivalent to the simple and greedy behaviours when $\AGa$ is arguing for $\arge$.}
    Similarly for when $\AGa$ is arguing against $\arge$, the shallow behaviour, by Definition \ref{def:shallow}, requires that $(\arga,\arge) \in \AttsA^{t-1}$ and $\nexists (\argb,\arge) \in \AttsA^{t-1} \setminus (\AttsI^{t-1} \cup \{ (\arga,\arge) \})$ such that $\SFa(\QBAFa^{t-1},\argb) > \SFa(\QBAFa^{t-1},\arga)$.
    Meanwhile, the greedy behaviour, by Point 1 of Definition \ref{def:greedy}, requires that $(\arga,\arge) \in \AttsA^{t-1}$, and by Point 2, requires that $\nexists (\argb,\arge) \in \AttsA^{t-1} \setminus (\AttsI^{t-1} \cup \{ (\arga,\arge) \})$ such that $\SFa(\QBAFa^{t-1},\argb) > \SFa(\QBAFa^{t-1},\arga)$. (Point 3 is not relevant here since $\argpaths((\argc,\arge)) = \{ (\argc,\arge) \}$ $\forall \argc \in \ArgsA^{t-1}$ at $t>0$.) The greedy behaviour is thus aligned with the shallow behaviour when $\AGa$ is arguing against $\arge$.
    %
    %\delete{Finally, the counterfactual behaviour, by Definition \ref{def:counterfactual}, requires that $(\arga,\arge) = argmax_{(\arga',\arge)  \in (\AttsA^{t-1}) \setminus ( \AttsI^{t-1})}\eff((\arga', \arge),\QBAFa^t)$, given that only attackers can increase $\arge$'s evaluation in this restricted case, where $\argpaths((\argc,\arge)) = \{ (\argc,\arge) \}$ $\forall \argc \in \ArgsA^{t}$. Given the definition of DF-QuAD, it is clear that, here, this is equivalent to the simple and greedy behaviours.}
    %
    %\delete{Thus, the three behaviours are equivalent both when $\arga$ is arguing for and against $\arge$ and so the proposition holds.}
    Thus, the proposition holds.
\end{proof}

Proof for Proposition \ref{propos:cf}:
\begin{proof}
    Since $E$ is unresolved and by Definition \ref{def:states}, let $\AGa \in \Agents$ be arguing for $\arge$ and $\AGb \in \Agents$ be arguing against $\arge$. 
    By Definition \ref{def:counterfactual}, we know that $\AGa$ contributed some $(\arga,\argb)$ such that $\eff((\arga, \argb),\QBAFa^t) > 0$. 
    Thus, by Proposition \ref{propos:paths}, $\Pros(\BAFi^t) \supset \Pros(\BAFi^{t-1})$ and, by Definition \ref{def:exchange}, $\Pros(\BAFi^n) \neq \emptyset$.
    Similarly, by Definition \ref{def:counterfactual}, we know that $\AGb$ contributed some $(\arga,\argb)$ such that $\eff((\arga, \argb),\QBAFb^t) < 0$. 
    Thus, by Proposition \ref{propos:paths}, $\Cons(\BAFi^t) \supset \Cons(\BAFi^{t-1})$ and, by Definition \ref{def:exchange}, $\Cons(\BAFi^n) \neq \emptyset$.
    Thus, conflict resolution is satisfied.
\end{proof}

 \subsection{Additional Propositions for §\ref{sec:behaviour}}

We now give more detail on the effects of varying the biases applied to learnt arguments by the machine.
The first case demonstrates the difficulty in achieving consensus between agents when one or all agents are incapable of learning.

\begin{proposition}\label{prop:unintelligent}
    Let, $\forall \AGa \in \Agents$, $\SFa$ be the DF-QuAD semantics.
    Then, for any $\AGa \in \Agents$ and $t>1$, if $\ArgsA^t = \ArgsA^{t-1} \cup \{ \arga \}$ where $\BSa^t(\arga) = 0$ and $\AttsA^t(\arga) \cup \SuppsA^t(\arga) = \emptyset$, then $\forall \argb \in \ArgsA^t \setminus \{ \arga \}$, $\SFa(\QBAFa^t,\argb) = \SFa(\QBAFa^{t-1},\argb)$.
\end{proposition}
\begin{proof}
    Since $\BSa^t(\arga) = 0$ and $\AttsA^t(\arga) \cup \SuppsA^t(\arga) = \emptyset$, by the definition of DF-QuAD (see §\ref{sec:preliminaries}), it must be the case that $\SFa(\QBAFa^t,\arga) = 0$. Also by the definition of DF-QuAD, an argument with an evaluation of zero has no effect on the arguments it attacks or supports, and thus any argument in $\AGa$'s private QBAF, i.e. $\SFa(\QBAFa^t,\argb) = \SFa(\QBAFa^{t-1},\argb)$ $\forall \argb \in \ArgsA^t \setminus \{ \arga \}$, thus the proposition holds. 
\end{proof}

We conjecture (but again leave to future work) that this behaviour is not limited to agents which evaluate arguments with DF-QuAD but also any semantics which ignore arguments with the minimum strength (e.g. as is described in \cite{Baroni_19} and by the property of \emph{neutrality} in \cite{Amgoud_17_BAF}).  %We leave more comprehensive studies of other semantics and the roles of properties in \AR{AX}s to future work.

%\todo{It can be seen that if agents' evaluation methods  and if agents assign a minimum bias on all learnt arguments, then such \AR{AX}s cannot be resolved.} 

The second case, meanwhile, guarantees that any learnt argument will have an effect all other downstream arguments' (including the explanandum's) strengths so long as no arguments in the private QBAF are assigned the minimium or maximum biases. %(Note that fewer restrictions are required for this effect if agents evaluate arguments with a semantics which satisfies \emph{strict monotonicity} \cite{Baroni_19} or \emph{strict bi-variate monotony/reinforcement} \cite{Amgoud_18}.)}

\begin{proposition}\label{prop:sceptical}
    %in partially sceptical agents with bias in all initial arguments between the extremities, every learnt argument is guaranteed to have some effect
    Let, $\forall \AGa \in \Agents$, $\SFa$ be the DF-QuAD semantics.
    Then, for any $\AGa \in \Agents$ and $t>1$ where $0<\BSa^t(\arga)<1$ $\forall \arga \in \ArgsA^t$, if $\ArgsA^t = \ArgsA^{t-1} \cup \{ \argb \}$, then $\forall \argc \in \ArgsA^t$ such that $|\argpaths(\argb,\argc)| = 1$, $\SFa(\QBAFa^t,\argc) \neq \SFa(\QBAFa^{t-1},\argc)$.
\end{proposition}
\begin{proof}
    It can be seen from the definition of the DF-QuAD semantics (see §\ref{sec:preliminaries}) that if $\nexists \arga \in \ArgsA^t$ such that $\BSa^t(\arga)=0$ or $\BSa^t(\arga)=1$, then $\nexists \arga \in \ArgsA^t$ such that $\SFa(\QBAFa^t,\arga)=0$ or $\SFa(\QBAFa^t,\arga)=1$.
    Then, also by the definition of DF-QuAD, it must be the case that if $\ArgsA^t = \ArgsA^{t-1} \cup \{ \argb \}$, then $\forall \argc \in \ArgsA^t$ such that $(\argb,\argc) \in \AttsA^t \cup \SuppsA^t$, $\SFa(\QBAFa^t,\argc) \neq \SFa(\QBAFa^{t-1},\argc)$. This same logic follows $\forall \argd \in \ArgsA^t$ such that $|\argpaths(\argb,\argd)| = 1$, and thus the proposition holds.
\end{proof}

Finally, the third case demonstrates the potential of incorporating credulity in machines with guarantees of rejection/weakening or acceptance/strengthening of arguments which are attacked or supported, \resp, by learnt arguments.

\begin{proposition}\label{prop:credulous}
    Let, $\forall \AGa \in \Agents$, $\SFa$ be the DF-QuAD semantics.
    Then, for any $\AGa \in \Agents$ and $t>1$, if $\ArgsA^t = \ArgsA^{t-1} \cup \{ \arga \}$ where $\BSa^t(\arga) = 1$ and $\AttsA^t(\arga) \cup \SuppsA^t(\arga) = \emptyset$, then for any $\argb \in \ArgsA^t$ such that $\arga \in \AttsA^t(\argb) \cup \SuppsA^t(\argb)$: 
    \begin{itemize}
        \item if $\arga \in \AttsA^t(\argb)$:
        \begin{itemize}
            \item if $\SuppsA^t(\argb) = \emptyset$ then $\SFa(\QBAFa^t,\argb) = 0$;
            \item if $\{ \argc \in \SuppsA^t(\argb) | \SFa(\QBAFa^t,\argc) = 1 \} = \emptyset$ then $\SFa(\QBAFa^t,\argb) < \BSa^t(\argb)$;
        \end{itemize}
        \item if $\arga \in \SuppsA^t(\argb)$:
        \begin{itemize}
            \item if $\AttsA^t(\argb) = \emptyset$ then $\SFa(\QBAFa^t,\argb) = 1$;
            \item if $\{ \argc \in \AttsA^t(\argb) | \SFa(\QBAFa^t,\argc) = 1 \} = \emptyset$ then $\SFa(\QBAFa^t,\argb) > \BSa^t(\argb)$.
        \end{itemize}
    \end{itemize}
    %in credulous agents, the presence of full strength attackers/supporters in the absence of: a.) full strength supporters/attackers or b.) supporters/attackers, gives a.) decreased/increased strength wrt bias or b.) minimum/maximum strength.
\end{proposition}
\begin{proof}
    Since $\BSa^t(\arga) = 1$ and $\AttsA^t(\arga) \cup \SuppsA^t(\arga) = \emptyset$, by the Definition of DF-QuAD (see §\ref{sec:preliminaries}), we can see that $\SFa(\QBAFa^t,\arga) = 1$. 
    Let $\arga \in \AttsA^t(\argb)$. If $\SuppsA^t(\argb) = \emptyset$ then, by the definition of DF-QuAD, $\SFa(\QBAFa^t,\argb) = 0$. Also by the definition of DF-QuAD, if $\{ \argc \in \SuppsA^t(\argb) | \SFa(\QBAFa^t,\argc) = 1 \} = \emptyset$, then $\SFa(\QBAFa^t,\argb) < \BSa^t(\argb)$.
    Similarly, we let $\arga \in \SuppsA^t(\argb)$. If $\AttsA^t(\argb) = \emptyset$ then, by the definition of DF-QuAD, $\SFa(\QBAFa^t,\argb) = 1$. Also by the definition of DF-QuAD, if $\{ \argc \in \AttsA^t(\argb) | \SFa(\QBAFa^t,\argc) = 1 \} = \emptyset$, then $\SFa(\QBAFa^t,\argb) > \BSa^t(\argb)$.
    Thus, the proposition holds. 
\end{proof}

%\begin{definition}\label{def:gullible} For any \emph{gullible machine agent} $\AGm$ at timestep $t$ and learnt argument $\arga \in \ArgsM^t \setminus \ArgsM^{t-1}$, $\BSm^t(\arga) = 1$. \end{definition}
%\begin{definition}\label{def:sceptical} For any \emph{sceptical machine agent} $\AGm$ at timestep $t$ and learnt argument $\arga \in \ArgsM^t \setminus \ArgsM^{t-1}$, $\BSm^t(\arga) = 0.5$. \end{definition}

%\delete{note about behaviour not being present in strict monotonicity-satisfying semantics?}

\subsection{Evaluation Measures for §\ref{sec:evaluation}}

The following gives the formal definitions for the first three of the evaluation measures used in the simulations. In all measures, when the denominator is zero, the measure is also.

\begin{definition}
Let $\mathcal{E}$ be a set of AXs for the same explanandum $\arge$ between agents $\Agents=\{\AGm, \AGh\}$, and let $\mathcal{R} \subseteq \mathcal{E}$ be the set of all resolved exchanges in $\mathcal{E}$. Then
\begin{itemize}
    \item the \emph{resolution rate} (\RR) of $\mathcal{E}$ is defined as 
    $\RR(\mathcal{E}) = 
    \frac{|\mathcal{R}|}{| \mathcal{E} |}$;
    \item the \emph{contribution rate} (\CR) of $\mathcal{E}$ is defined as 
    $\CR(\mathcal{E}) = \frac{\sum_{ E \in %\{ E \in \mathcal{E} | E \text{ is resolved} \}
    \mathcal{R}}{E_{\#}}}
    {| %\{ E \in \mathcal{E} | E \text{ is resolved} \} 
    \mathcal{R}|}$, where for  %any $E \in \mathcal{E}$ such that 
    $E = \langle %\BAFu, 
    \BAFi^0, \ldots, \BAFi^{n}, \Agents^0, \ldots, \Agents^n, \SpeakerM \rangle$, $E_{\#} = | \AttsI^n \cup \SuppsI^n |$;
    \item the  \emph{persuasion rate} (\PR) of $\AGa \in \Agents$ over $\mathcal{E}$ is defined as 
    $\PR(\AGa,\mathcal{E}) = 
    \frac{|\{ E \in \mathcal{R} | \forall \AGb \in \Agents, \StanceB(\QBAFb^n, \arge) = \StanceA(\QBAFa^0, \arge) \} |}
    {|\mathcal{R}|}$. 
\end{itemize}
\end{definition}

Before giving the final measure we let the following indicate the set of arguments which would have had the biggest effect on the explanandum in \emph{the other agent}'s private QBAF: 
\begin{align}
\text{ if }& \AGa \text{ is arguing for } \arge, \text{ then:  } 
\SpeakerM_{max}(\AGa,t) =  \nonumber \\ 
\{ \! (\!\arga, & \argb\!) \!\in\! argmax_{(\arga'\!, \argb') \in (\AttsA^{t} \! \cup \SuppsA^{t}) \!\setminus \! (\AttsI^{t} \! \cup \SuppsI^{t})} \SFb(\QBAFb^{t+1}\!\!,\arge) | \ArgsB^{t+1} \!\!=\! \ArgsB^{t} \!\cup\! \{ \arga' \}\! \} \nonumber \\
\text{ if }& \AGa \text{ is arguing against } \arge, \text{ then:  }
\SpeakerM_{max}(\AGa,t) =  \nonumber \\ 
\{ \! (\!\arga, & \argb\!) \!\in\! argmin_{(\arga'\!, \argb') \in (\AttsA^{t} \! \cup \SuppsA^{t}) \!\setminus \! (\AttsI^{t} \! \cup \SuppsI^{t})} \SFb(\QBAFb^{t+1}\!\!,\arge) | \ArgsB^{t+1} \!\!=\! \ArgsB^{t} \!\cup\! \{ \arga' \}\! \} \nonumber 
\end{align}

We can then define the contribution accuracy.

\begin{definition} 
Let $\mathcal{E}$ be a set of AXs for the same explanandum $\arge$, each between two agents $\Agents=\{\AGm, \AGh\}$. Then, the \emph{contribution accuracy} (\CA) of $\AGa \in \Agents$ over $\mathcal{E}$ is defined as $\CA(\AGa, \mathcal{E}) = \frac{\sum_{ E \in \mathcal{E}}{acc(\AGa,E)}}{| \mathcal{E} |}$, 
where for $E = \mathcal{E}$, $acc(\AGa, E) = 0$ if $\{ (a,b) \in \AttsI^n\cup \SuppsI^n | \SpeakerM((a,b))=(\AGa,t) \} = \emptyset$, otherwise: %\langle \BAFi^0, \ldots, \BAFi^{n}, \Agents^0, \ldots, \Agents^n, \SpeakerM \rangle$: 
$$acc(\AGa, E) = 
\frac{|\{ (a,b) \in \SpeakerM_{max}(\AGa,t) | \SpeakerM((a,b))=(\AGa,t) \} |}
{|\{ (a,b) \in \AttsI^n\cup \SuppsI^n | \SpeakerM((a,b))=(\AGa,t) \} | }$$.
\end{definition}

\end{document}